\documentclass{article}

\usepackage{fullpage}
\usepackage{url,color,cite}
\usepackage{caption}
\usepackage{amsmath,amsthm,amsfonts,amssymb}
\usepackage{cases}
\usepackage{mathtools}

\usepackage{float}
\usepackage{epstopdf}
\usepackage{nicefrac}
\usepackage{algorithm}
\usepackage[noend]{algpseudocode}
\usepackage{pifont}
\usepackage{subcaption}
\usepackage{sidecap}
\usepackage{multirow}
\usepackage{multicol}
\usepackage{bbm}

\newtheorem{lemma}{Lemma}
\newtheorem*{lemma*}{Lemma}
\newtheorem{theorem}{Theorem}
\newtheorem*{theorem*}{Theorem}
\theoremstyle{definition}
\newtheorem{defn}{Definition}
\theoremstyle{definition}

\newtheorem{remark}{Remark}
\theoremstyle{definition}
\newtheorem{corollary}{Corollary}
\theoremstyle{definition}

\newtheorem*{claim*}{Claim}

\newcommand{\bbE}{\mathbb{E}}
\newcommand{\bbN}{\mathbb{N}}
\newcommand{\bbR}{\mathbb{R}}
\newcommand{\calG}{\mathcal{G}}
\newcommand{\calA}{\mathcal{A}}
\newcommand{\calB}{\mathcal{B}}
\newcommand{\calC}{\mathcal{C}}
\newcommand{\calD}{\mathcal{D}}
\newcommand{\calE}{\mathcal{E}}
\newcommand{\calI}{\mathcal{I}}
\newcommand{\calH}{\mathcal{H}}
\newcommand{\calM}{\mathcal{M}}

\newcommand{\calO}{\mathcal{O}}
\newcommand{\calP}{\mathcal{P}}
\newcommand{\calQ}{\mathcal{Q}}
\newcommand{\calR}{\mathcal{R}}
\newcommand{\calS}{\mathcal{S}}
\newcommand{\calT}{\mathcal{T}}
\newcommand{\calU}{\mathcal{U}}
\newcommand{\calV}{\mathcal{V}}
\newcommand{\calX}{\mathcal{X}}
\newcommand{\calY}{\mathcal{Y}}
\newcommand{\calZ}{\mathcal{Z}}

\newcommand{\frS}{\mathfrak{S}}

\renewcommand{\(}{\left(}
\renewcommand{\)}{\right)}

\newcommand{\eps}{\epsilon}
\newcommand{\epsbar}{\overline{\epsilon}}
\newcommand{\delbar}{\overline{\delta}}
\newcommand{\samp}{\mathrm{samp}}
\newcommand{\bmu}{\boldsymbol{\mu}}
\newcommand{\be}{\boldsymbol{e}}

\newcommand{\bu}{\boldsymbol{u}}
\newcommand{\bv}{\boldsymbol{v}}
\newcommand{\bq}{\boldsymbol{q}}
\newcommand{\bx}{\boldsymbol{x}}
\newcommand{\by}{\boldsymbol{y}}

\newcommand{\bz}{\boldsymbol{z}}

\newcommand{\hatx}{\widehat{\bx}}
\newcommand{\barx}{\overline{\bx}}
\newcommand{\hatcalP}{\widehat{\mathcal{P}}}

\newcommand{\bH}{\mathbf{H}}
\newcommand{\bI}{\mathbf{I}}

\usepackage{microtype}
\usepackage{graphicx}
\usepackage{booktabs} 
\usepackage{hyperref}
\usepackage{cleveref}

\usepackage{lipsum}
\allowdisplaybreaks

\title{Shuffled Model of Federated Learning: Privacy, Communication, and Accuracy Trade-offs}

\author{Antonious M. Girgis, Deepesh Data, Suhas Diggavi, \\ Peter  Kairouz, and Ananda Theertha Suresh \thanks{Antonious M. Girgis, Deepesh Data and Suhas Diggavi are with the University of California, Los Angeles, USA.
Email: amgirgis@g.ucla.edu, deepesh.data@gmail.com, suhas@ee.ucla.edu.
Peter  Kairouz and Ananda Theertha Suresh are with Google Research, USA.
Email: kairouz@google.com, theertha@google.com.
}
}

\date{}

\begin{document}
\maketitle

\begin{abstract}

We consider a distributed empirical risk minimization (ERM)
optimization problem with communication efficiency and privacy
requirements, motivated by the federated learning (FL) framework
\cite{kairouz2019advances}. Unique challenges to the traditional ERM
problem in the context of FL include {\sf (i)} need to provide privacy
guarantees on clients' data, {\sf (ii)} compress the
communication between clients and the server, since clients might have
low-bandwidth links, {\sf (iii)} work with a dynamic client population
at each round of communication between the server and the clients, as
a small fraction of clients are sampled at each round. To address
these challenges we develop (optimal) communication-efficient schemes
for private mean estimation for several $\ell_p$ spaces, enabling efficient gradient aggregation
for each iteration of the optimization solution of the ERM.  
We also provide lower and upper bounds for mean estimation with privacy and communication 
constraints for arbitrary $\ell_p$ spaces.
To get the overall communication, privacy, and optimization performance
operation point, we combine this with privacy amplification
opportunities inherent to this setup. Our solution takes advantage of the
inherent privacy amplification provided by client sampling and data
sampling at each client (through Stochastic Gradient Descent) as well as
the recently developed privacy framework using anonymization, which
effectively presents to the server responses that are randomly
shuffled with respect to the clients. Putting these together, we
demonstrate that one can get the same privacy,
optimization-performance operating point developed in recent methods
that use full-precision communication, but at a much lower
communication cost, \emph{i.e.,} effectively getting communication
efficiency for ``free''.

\end{abstract}

\section{Introduction}\label{sec:intro}

In this paper we consider a federated learning (FL) framework
\cite{konevcny2016federated,yang2019federated,kairouz2019advances}, where the data is generated across $m$
clients. The server wants to learn a machine learning
model that minimizes a certain objective function using the $m$ local
datasets, without collecting the data at the central server due to
privacy considerations. Specifically, each client $i$ has a local
dataset $\calD_i=\lbrace d_{i1},\ldots,d_{ir}\rbrace \subset \frS^r$
comprising $r$ data points, where $\frS$ is the set from which the
$i$'th client's data is from.\footnote{The data could be images with labels,
  \emph{e.g.,} $8\times 8$ pixel blocks with labels, where each pixel is represented by
  $32$ bits and each label is represented by an integer from $\lbrace 1\ldots,10\rbrace$, in which case $\frS=\mathbb{F}^{64}\times\mathbb{G} $, where
  $\mathbb{F}=\{1,\ldots,256\}$ and $\mathbb{G}=\lbrace 1,\ldots,10\rbrace$.  
  Another example is the text represented by
  words, in which case $\frS=\mathcal{W}^*$, where $\mathcal{W}$ is
  the language alphabet and $\frS$ are strings of letters from the
  alphabet.}
The server wants to solve the following empirical risk minimization problem:
\begin{equation}~\label{eq:EMP}
\arg\min_{\theta\in\mathcal{C}} \( F(\theta) := \frac{1}{m}\sum_{i=1}^{m} F_i(\theta) \).
\end{equation}
Here, $\mathcal{C}\subset \mathbb{R}^d$ is a
closed convex set, and $F_i(\theta)$ is a local loss function dependent on the local
dataset $\calD_i$ at client $i$ evaluated at the model parameters
$\theta\in\bbR^d$; see Section~\ref{subsec:optimization-results} for more details on the problem setup.
In order to generate a learning model using
\eqref{eq:EMP}, the commonly used mechanism is Stochastic Gradient
Descent (SGD) \cite{bottou2010large}. Federated learning (FL) introduces
several unique challenges to this traditional model that cause tension
with the objective in \eqref{eq:EMP}: {\sf (i)} we need to
provide privacy guarantees on the locally residing data $\calD_i$ at client $i$, as the
data not only needs to be \emph{remain} at the clients but additionally
needs to kept private according to certain
requirements/guarantees; {\sf (ii)} compress (as efficiently as
possible) the communication between clients and the server, since the
clients may connect with low-bandwidth (wireless) links; and {\sf (iii)}
work with a dynamic client population in each round of communication
between the server and the clients. This happens due to scale
(\emph{e.g.,} tens of millions of devices) and only a small fraction of
clients are sampled at each communication round depending on their
availability.

These requirements make the problem challenging, especially when one
wants to give strong privacy guarantees while training models that
give good learning performance. Since we need to give privacy to the
local data residing at the clients, the traditional framework to give
guarantees is through the notion of local differential privacy, where
the server is itself untrusted. The challenge is that the traditional
privacy approach to the learning problem uses local differential
privacy (LDP) \cite{warner1965randomized,evfimievski2004privacy,duchi2013local,beimel2008distributed,kasiviswanathan2011can}, which is known to
give poor learning performance~\cite{duchi2013local,kasiviswanathan2011can,kairouz2016discrete}.

In recent works, a new privacy framework using anonymization has been
proposed in the so-called \emph{shuffling model} \cite{erlingsson2019amplification,ghazi2019power,balle2019improved,ghazi2019scalable,balle2019differentially,ghaziprivate,cheu2019distributed,balle2019privacy,balle2020private}. This
model enables significantly better privacy-utility performance by
amplifying privacy (scaling with number of clients as
$\frac{1}{\sqrt{m}}$ with respect to LDP) through this anonymization,
which effectively presents the central server with responses which are
randomly shuffled with respect to the clients, providing additional
privacy. Another mechanism to amplify privacy is through randomized
sampling \cite{beimel2010bounds,kasiviswanathan2011can,Jonathan2017sampling}. This naturally arises in the considered SGD framework, since clients do mini-batch sampling of local
data and also there is sampling of clients themselves in each iteration,
as in the federated learning framework \cite{konevcny2016federated,yang2019federated,kairouz2019advances}.

In this paper, we enable privacy amplification for the FL problem using
both forms of amplification: shuffling and sampling (data and
clients). Note that privacy amplification by subsampling (both data and
clients) happens automatically\footnote{In this paper, we use an
  abstraction for the federated learning model, where clients are
  sampled randomly. In practice, there are many more complicated
  considerations for sampling, including availability, energy usage,
  time-of-day etc., which we do not model.}, and we quantify that in
this paper, while the secure shuffling (anonymization) is performed
explicitly which adds an additional layer of privacy that allows
transferring the local privacy guarantees to central privacy
guarantees.

Another important aspect is that of requiring communication efficiency
instantiated through compression of the gradients computed by each
active client. There has been a significant recent progress in this
topic (see  \cite{alistarh2017qsgd,pmlr-v97-karimireddy19a,wen2017terngrad,stich2018sparsified,alistarh2018convergence,koloskova2019decentralized,stich2018sparsified,basu2019qsparse,singh2019sparq,singh2020squarm} and references therein). However, there has been less work in combining privacy and compression in the
optimization/learning framework of~\eqref{eq:EMP}, with the notable exception of
\cite{agarwal2018cpsgd}, which we will elaborate on soon. One question that
arises is whether one pays a price to do compression in terms of the
privacy-performance trade-off; a question we address in this paper.

In this paper we (partially) solve the main problem of privately
learning a model with compressed communication, with good learning
performance while giving strong guarantees on privacy. We believe that
this is the first result that analyses the optimization performance
with schemes devised using compressed gradient exchange, mini-batch SGD
while giving data privacy guarantees for clients using a shuffled
framework. Our main contributions are as follows.
\begin{itemize}
\item We prove that one can get communication efficiency ``for free''
  by demonstrating schemes that use $O(\log d)$ bits per client
  (for several cases) to obtain the same privacy-performance operating
  point achieved by full precision gradient exchange.\footnote{Our work focuses on symmetric, private-randomness mechanisms. We do not assume the existence of public randomness in this work as we use the shuffling model.} 
  We do this using the shuffled privacy model and amplification by sampling (client data
  through mini-batch SGD and clients themselves in federated
  sampling). Note that sampling of clients and data points together give non-uniform sampling of data points,
  so we cannot use the existing results on amplification by subsampling. We instead give one privacy proof that combines both sampling and shuffling techniques together and analyze the total privacy gain.
\item At each round of the iterative optimization, one needs to
  privately aggregate the gradients in a communication efficient
  manner. To do this, we develop new private, compressed mean
  estimation techniques in a minimax estimation framework, that are
  (order optimal) under several $\ell_p$ geometries for the
  vectors. We develop both lower bounds and matching schemes for this
  problem. These results may also be of independent interest (see
  Section \ref{sec:PrivMeanEst}).
\item In order to complete the overall optimization-performance
  trade-off with privacy and communication, we use the amplification
  by shuffling framework of~\cite{erlingsson2019amplification,balle2019privacy} adapted to our setup
  where we have mini-batch SGD, compression and client sampling (see Section \ref{sec:OptPerf}).
\end{itemize}
We will put our contributions in context to the existing literature
next.

\subsection{Related Work}

Among the several main challenges in the recently developed FL
framework (see \cite{kairouz2019advances} and references therein), we focus in
this paper on the combination of privacy and communication efficiency, and
examining its impact on model learning.  We briefly review some of the main developments in related papers on these topics below.

\subsubsection{Communication-Privacy Trade-offs}
Distributed mean estimation and its use in training learning models has been studied extensively in the literature (see~\cite{suresh2017distributed,alistarh2017qsgd,gandikota2019vqsgd,mayekar2020limits} and references therein). In~\cite{suresh2017distributed}, the authors have proposed a communication efficient scheme for estimating the mean of set a of vectors distributed over multiple clients. In~\cite{acharya2019hadamard}, Acharya et. al. studied the discrete distribution estimation under LDP. They proposed a randomized mechanism based on Hadamard coding which is optimal for all privacy regime and requires $\calO\left(\log\left(d\right)\right)$ bits per client, where $d$ denotes the support size of the discrete distribution. In~\cite{acharya2019communication}, the authors consider both private and public coin mechanisms, and show that the Hadamard mechanism is near optimal in terms of communication for both distribution and frequency estimation. However, the LDP mechanisms suffer from the utility degradation that motivates other work to find alternative techniques to improve the utility under LDP. One of new developments in privacy is the use of anonymization to amplify the privacy by using secure shuffler.  In~\cite{cheu2019distributed,balle2019privacy,balle2020private}, the authors studied the mean estimation problem under LDP with secure shuffler, where they show that the shuffling provides better utility than the LDP framework without shuffling. 

\subsubsection{Private Optimization}
In~\cite{chaudhuri2011differentially}, Chaudhuri et al.\ studied
\emph{centralized} privacy-preserving machine learning algorithms for
convex optimization problem. The authors
proposed a new idea of perturbing the objective function to preserve
privacy of the training dataset.  In~\cite{bassily2014private},
Bassily et al.\ derived lower bounds on the empirical risk
minimization under \emph{central} differential privacy
constraints. Furthermore, they proposed a differential privacy SGD
algorithm that matches the lower bound for convex functions. In~\cite{abadi2016deep}, the
authors have generalized the private SGD algorithm proposed
in~\cite{bassily2014private} for non-convex optimization framework. In
addition, the authors have proposed a new analysis technique, called
moment accounting, to improve on the strong composition theorems to
compute the central differential privacy guarantee for iterative
algorithms. However, the works
mentioned,~\cite{chaudhuri2011differentially,bassily2014private,abadi2016deep},
assume that there exists a trusted server that collects the clients'
data. This motivates other works to design a distributed SGD
algorithms, where each client perturbs her own data without needing a
trusted server. For this, the natural privacy framework is \emph{local}
differential privacy or LDP (\emph{e.g.,} see
\cite{warner1965randomized,duchi2013local,evfimievski2004privacy,bhowmick2018protection}). However, it is well understood that LDP does not give good performance guarantees as it requires significant
local randomization to give privacy guarantees
\cite{duchi2013local,kasiviswanathan2011can,kairouz2016discrete}. The two most related papers to our work are
~\cite{erlingsson2020encode,agarwal2018cpsgd} which we describe below.

In~\cite{erlingsson2020encode}, the authors have proposed a
distributed local-differential-privacy gradient descent algorithm,
where each client has one sample. In their proposed algorithm, each client
perturbs the gradient of her sample using an LDP mechanism. To
improve upon the LDP performance guarantees, they use the newly
proposed anonymization/shuffling framework
\cite{balle2019privacy}. Therefore in their work, gradients of all clients
are passed through a secure shuffler that eliminates the identities of
the clients to amplify the central privacy guarantee. However, their
proposed algorithm is not communication efficient, where each client has
to send the full-precision gradient without compression. Our work is
different from~\cite{erlingsson2020encode}, as we propose a
communication efficient mechanism for each client that requires $O(\log
d)$ bits per client, which can be significant for large
$d$. Furthermore, our algorithm consider multiple data samples at
client, which is accessed through a mini-batch random sampling at each
iteration of the optimization.  This requires a careful combination of
compression and privacy analysis in order to preserve the variance
reduction of mini-batch as well as privacy.\footnote{The naive method
  of quantizing the aggregated mini-batch gradient will fail to
  preserve the required variance reduction.} In addition we obtain a
gain in privacy by using the fact that (anonymized) clients are sampled
(\emph{i.e.,} not all clients are selected at each iteration) as
motivated by the federated learning framework.

\cite{agarwal2018cpsgd} proposed a communication-efficient algorithm
for learning models with central differential privacy. Let $n$ be the
number of clients per round and $d$ be the dimensionality of the
parameter space. They proposed cp-sgd, a communication efficient
algorithm, where clients need to send $O(\log(1 + \frac{d}{n}
\epsilon^2) + \log \log \log \frac{nd}{\epsilon\delta} )$ bits of
communication \emph{per coordinate} \emph{i.e.,} $O\left
(d\left\{\log(1 + \frac{d}{n} \epsilon^2) + \log \log \log
\frac{nd}{\epsilon\delta} \right \}\right )$ bits per round to
achieve the same local differential privacy guarantees of $\epsilon_0$
as the Gaussian mechanism. Their algorithm is based on a Binomial
noise addition mechanism and secure aggregation. In contrast, we
propose a generic framework to convert any LDP algorithm to a central
differential privacy guarantee and further use recent results on
amplification by shuffling, that also achieves better compression in
terms of number of bits per client. \\

\noindent{\textbf{Organization.} The paper is organized as follows. In
  Section \ref{sec:prelims}, we set up the problem and notation, while
  giving preliminary background results on privacy amplification
  through shuffling and sampling.  We provide the main results of the
  paper in Section \ref{sec:MainRes} and also give some
  interpretations.  In Section \ref{sec:PrivMeanEst} we analyze
  private vector minimax mean estimation for various geometrical
  constraints, applicable to gradient aggregation for optimization;
  providing schemes and impossibility results.  In Section
  \ref{sec:OptPerf} examine the communication-privacy and
  optimization-performance trade-offs of our schemes, putting together
  the results from Section \ref{sec:PrivMeanEst} to give the proof of
  the main theorem~\ref{thm:main-opt-result}. We conclude with a brief discussion
  in Section \ref{sec:Disc}.

\section{Preliminaries}
\label{sec:prelims} 

In this section, we state some preliminary definitions that we use
throughout the paper and also state some results from literature.  We
state the formal definitions of (local) differential privacy (DP) in
Section~\ref{subsec:diff-priv} and strong composition theorem for DP
in Section \ref{subsec:StrComp}.
As mentioned in Section~\ref{sec:intro}, we use subsampling and
shuffling techniques for privacy amplification and we describe them in
Section~\ref{subsec:privacy-amplification}.  Finally, we present one
of our main ingredients in the proposed compressed and private SGD
algorithm, which is a method of private mean estimation using
compressed updates, in Section~\ref{subsec:cp-mean-est}.  We use this
formulation to study the problem in the minimax framework and derive
upper and lower bounds in a variety of settings in Section
\ref{sec:PrivMeanEst}.

\subsection{Differential Privacy}\label{subsec:diff-priv}
In this section, we formally define local differential privacy (LDP) and (central) differential privacy (DP).
First we recall the standard definition of LDP \cite{kasiviswanathan2011can}.
\begin{defn}[Local Differential Privacy - LDP \cite{kasiviswanathan2011can}]\label{defn:LDPdef}
For $\epsilon_0\geq0$ and $b\in\bbN^+:=\{1,2,3,\hdots\}$, a randomized mechanism $\calR:\calX\to\calY$ is said to be $\eps_0$-local differentially private (in short, $\eps_{0}$-LDP), if for every pair of inputs $\bx,\bx'\in\calX$, we have 
\begin{equation}~\label{ldp-def}
\Pr[\calR(\bx)=\by] \leq \exp(\eps)\Pr[\calR(\bx')=\by], \qquad \forall \by\in\calY.
\end{equation}
\end{defn}

In our problem formulation, since each client has a communication budget on what it can send in each SGD iteration while keeping its data private, it would be convenient for us to define two parameter LDP with privacy and communication budget.
\begin{defn}[Local Differential Privacy with Communication Budget - CLDP]\label{defn:local-DP}
For $\epsilon_0\geq0$ and $b\in\bbN^+$, a randomized mechanism $\calR:\calX\to\calY$ is said to be $(\eps_0,b)$-communication-limited-local differentially private (in short, $(\eps_{0},b)$-CLDP), if for every pair of inputs $\bx,\bx'\in\calX$, we have 
\begin{equation}~\label{ldp-comm_def}
\Pr[\calR(\bx)=\by] \leq \exp(\eps)\Pr[\calR(\bx')=\by], \qquad \forall \by\in\calY.
\end{equation}
Furthermore, the output of $\calR$ can be represented using $b$ bits. 
\end{defn}
Here, $\epsilon_{0}$ captures the privacy level, lower the
$\epsilon_0$, higher the privacy.  When we are not concerned about the
communication budget, we succinctly denote the corresponding
$(\eps_{0},\infty)$-CLDP, by its correspondence to the classical LDP as
$\eps_0$-LDP \cite{kasiviswanathan2011can}.

Let $\calD=\lbrace \bx_1,\ldots,\bx_n\rbrace$ denote a dataset comprising $n$ points from $\calX$. We say that two datasets $\calD=\lbrace \bx_1,\ldots,\bx_n\rbrace$ and $\calD^{\prime}=\lbrace \bx_1^{\prime},\ldots,\bx_n^{\prime}\rbrace$ are neighboring if they differ in one data point. In other words, $\calD$ and $\calD^{\prime}$ are neighboring if there exists an index $i\in\left[n\right]$ such that $\bx_i\neq \bx_i^{\prime}$ and $\bx_j=\bx_j^{\prime}$ for all $j\neq i$. 
\begin{defn}[Central Differential Privacy - DP \cite{Calibrating_DP06,dwork2014algorithmic}]\label{defn:central-DP}
For $\epsilon,\delta\geq0$, a randomized mechanism $\calM:\calX^n\to\calY$ is said to be $\left(\epsilon,\delta\right)$-differentially private (in short, $\left(\epsilon,\delta\right)$-DP), if for all neighboring datasets $\calD,\calD^{\prime}\in\calX^{n}$ and every subset $\calE\subseteq \calY$, we have
\begin{equation}~\label{dp_def}
\Pr\left[\calM\left(\calD\right)\in\calE\right]\leq \exp(\epsilon)\Pr\left[\calM\(\calD^{\prime}\)\in\calE\right]+\delta.
\end{equation}
\end{defn}
\begin{remark}
For any $\epsilon_0$-LDP mechanism $\calR:\calX\to\calY$, it is easy to verify that the randomized mechanism $\calM:\calX^n\to\calY$ defined by $\calM\(\bx_1,\hdots,\bx_n\):=\(\calR\(\bx_1\),\ldots,\calR\(\bx_n\)\)$ is $(\epsilon_0,0)$-DP.
\end{remark}
\begin{remark}
Note that in this paper we make a clear distinction between the
notation used for central differential privacy, denoted by
$\left(\epsilon,\delta\right)$-DP (see Definition
\ref{defn:central-DP}), local differential privacy $\eps_{0}$-LDP (see
definition \ref{defn:LDPdef}) and communication limited local
differential privacy, denoted by $(\eps_{0},b)$-CLDP (see Definition
\ref{defn:local-DP}).
\end{remark}
The main objective of this paper is to make SGD differentially private and communication-efficient, suitable for federated learning.
For that we compress and privatize gradients in each SGD iteration.
Since the parameter vectors in any iteration depend on the previous iterations, so do the gradients,
which makes this procedure a sequence of many adaptive DP mechanisms.
We can calculate the final privacy guarantees achieved at the end of this procedure by using composition theorems.

\subsection{Strong Composition \cite{Boosting-DP_DworkRV10}}
\label{subsec:StrComp}
Let $\calM_1\(\calI_1,\calD\),\ldots,\calM_T\(\calI_T,\calD\)$ be a sequence of $T$ adaptive DP mechanisms, where $\calI_i$ denotes the auxiliary input to the $i$th mechanism, which may depend on the previous mechanisms' outputs and the auxiliary inputs $\{(\calI_j,\calM_j(\calI_j,\calD)):j<i\}$.
There are different composition theorems in literature to analyze the privacy guarantees of the composed mechanism $\calM(\calD)= \(\calM_1\(\calI_1,\calD\),\ldots,\calM_T\(\calI_T,\calD\)\)$. \\

Dwork et al.~\cite{Boosting-DP_DworkRV10} provided a strong composition theorem (which is stronger than the basic composition theorem in which the privacy parameters scale linearly with $T$) where the privacy parameter of the composition mechanism scales as $\sqrt{T}$ with some loss in $\delta$. Below, we provide a formal statement of that result from \cite{dwork2014algorithmic}.
\begin{lemma}[Strong Composition~{\cite[Theorem~$3.20$]{dwork2014algorithmic}}]\label{lemm:strong-composition}
Let $\calM_1,\hdots,\calM_T$ be $T$ adaptive $(\epsbar,\delbar)$-DP mechanisms, where $\overline{\epsilon},\overline{\delta}\geq0$.
Then, for any $\delta'>0$, the composed mechanism $\calM=(\calM_1,\hdots,\calM_T)$ is $(\epsilon,\delta)$-DP, where
\[\epsilon=\sqrt{2T\log\left(1/\delta^{\prime}\right)}\overline{\epsilon}+T\overline{\epsilon}\left(e^{\overline{\epsilon}}-1\right),\quad \delta=T\overline{\delta}+\delta^{\prime}.\]
In particular, when $\overline{\epsilon}=\mathcal{O}\left(\sqrt{\frac{\log\left(1/\delta^{\prime}\right)}{T}}\right)$, we have $\epsilon=\mathcal{O}\left(\overline{\epsilon}\sqrt{T\log\left(1/\delta^{\prime}\right)}\right)$.
\end{lemma}
Note that training large-scale machine learning models (e.g., in deep learning) typically requires running SGD for millions of iterations, as the dimension of the model parameter is quite large. We can make it differentially private by adding noise to the gradients in each iteration, and appeal to the strong composition theorem to bound the privacy loss of the entire process (which in turn dictates the amount of noise to be added in each iteration).

\subsection{Privacy Amplification}\label{subsec:privacy-amplification}
In this section, we describe the techniques that can be used for privacy amplification.
The first one amplifies privacy by subsampling the data (to compute stochastic gradients) as well as the clients (as in FL), and the other one amplifies privacy by shuffling. \\

\subsubsection{Privacy Amplification by Subsampling}\label{subsubsec:amplification_sampling}
Suppose we have a dataset $\calD'=\lbrace U_1,\hdots, U_{r_1}\rbrace\in\calU^{r_1}$ consisting of $r_1$ elements from a universe $\calU$. A subsampling procedure takes a dataset $\calD'\in\calU^{r_1}$ and subsamples a subset from it as formally defined below.

\begin{defn}[Subsampling] 
The subsampling operation $\samp_{r_1,r_2}:\calU^{r_1}\to\calU^{r_2}$ takes a dataset $\calD'\in\calU^{r_1}$ as input and selects uniformly at random a subset $\calD''$ of $r_2\leq r_1$ elements from $\calD'$. 
Note that each element of $\calD'$ appears in $\calD''$ with probability $q=\frac{r_2}{r_1}$.
\end{defn}
The following result states that the above subsampling procedure amplifies the privacy guarantees of a DP mechanism.
\begin{lemma}[Amplification by Subsampling~\cite{kasiviswanathan2011can}]\label{lemma:amplification_sampling}
Let $\calM:\calU^{r_2}\to\calV$ be an $(\eps,\delta)$-DP mechanism. Then, the mechanism $\calM':\calU^{r_1}\to\calV$ defined by $\calM'=\calM\circ \samp_{r_1,r_2}$ is $(\eps',\delta')$-DP, where $\eps'=\log(1+q(e^{\eps}-1))$ and $\delta'=q\delta$ with $q=\frac{r_2}{r_1}$.
In particular, when $\eps<1$, $\calM'$ is $(\calO(q\eps),q\delta)$-DP. 
\end{lemma}

Note that in the case of subsampling the data for computing stochastic
gradients, where client $i$ selects a mini-batch of size $s$ from its
local dataset $\calD_i$ that has $r$ data points, we take
$\calD'=\calD_i$, $r_1=r$, and $r_2=s$.  In the case of subsampling
the clients, $k$ clients are randomly selected from the $m$ clients,
we take $\calD'=\{1,2,\hdots,m\}$, $r_1=m$, and $r_2=k$. An important
point is that such a sub-sampling is not uniform overall 
(i.e., this does not imply that any subset of $ks$ data points is chosen with equal probability) 
and we cannot directly apply the above result. We need to revisit the proof of Lemma~\ref{lemma:amplification_sampling}
to adapt it to our case, and we do it in Lemma~\ref{lemm:priv_opt}, which is proved in
Appendix \ref{app:supp_composition_priv}. In fact, the proof of Lemma~\ref{lemm:priv_opt} is more general than just adapting the amplification by subsampling to our setting, it also incorporates the amplification by shuffling, which is crucial for obtaining strong privacy guarantees. We describe it next.

\subsubsection{Privacy Amplification by Shuffling}\label{subsubsec:amplification_shuffling}
Consider a set of $m$ clients, where client $i\in[m]$ has a data $\bx_i\in\calX$. 
Let $\calR:\calX\to\calY$ be an $\eps_0$-LDP mechanism. 
The $i$-th client applies $\calR$ on her data $\bx_i$ to get a private message $\by_i=\calR(\bx_i)$. There is a secure shuffler $\calH_m:\calY^{m}\to \calY^{m}$ that receives the set of $m$ messages $(\by_1,\ldots,\by_m)$ and generates the same set of messages in a uniformly random order. 

The following lemma states that the shuffling amplifies the privacy of an LDP mechanism by a factor of $\frac{1}{\sqrt{m}}$.

\begin{lemma}[Amplification by Shuffling]\label{lemma:amplification_shuffling}
Let $\mathcal{R}$ be an $\epsilon_0$-LDP mechanism. Then, the mechanism $\calM(\bx_1,\hdots,\bx_m):=\calH_m\circ(\calR(\bx_1),\ldots,\calR(\bx_m))$ satisfies $(\eps,\delta)$-differential privacy, where
\begin{enumerate}
\item {\em \cite[Corollary~$5.3.1$]{balle2019privacy}.} 
If $\eps_0\leq \frac{\log\(m/\log\left(1/\delta\right)\)}{2}$, then for any $\delta>0$, we have \\ $\eps=\calO\left(\min\{\eps_0,1\}e^{\eps_0}\sqrt{\frac{\log\left(1/\delta\right)}{m}}\right)$. 
\item {\em \cite[Corollary~$9$]{erlingsson2019amplification}.} 
If $\eps_0<\frac{1}{2}$, then for any $\delta\in(0,\frac{1}{100})$ and $m\geq 1000$, we have $\eps=12\eps_0\sqrt{\frac{\log\left(1/\delta\right)}{m}}$.
\end{enumerate}
\end{lemma}
In our proposed algorithm, only $k\leq m$ clients send messages and each client sends a mini-batch of $s$ gradients.
So, in total, shuffler applies the shuffling operation on $ks$ gradients.
In our algorithm, though sampling and shuffling are applied one after another (first $k$ clients are sampled, then each client samples $s$ data points, and then shuffling of these $ks$ data points is performed), we analyze the privacy amplification we get using both of these techniques by analyzing them together; see Lemma \ref{lemm:priv_opt} proved in Appendix \ref{app:supp_composition_priv}.

\subsection{Compressed and Private Mean Estimation via Minimax Risk}\label{subsec:cp-mean-est}
Recall that in each SGD iteration, server sends the current parameter vector to all clients, upon receiving which they compute stochastic gradients from their local datasets and send them to the server, who then computes the average/mean of received gradients and updates the parameter vector. Note that these gradients (over the entire execution of algorithm) may also leak information about the datasets. As mentioned in Section~\ref{sec:intro}, we also compress the gradients to mitigate the communication bottleneck.

In this section, we formulate the generic mimimax estimation framework
for mean estimation of a given set of $n$ vectors that preserves
privacy and is also communication-efficient. We then apply that method
at the server in each SGD iteration for aggregating the gradients. We
derive upper and lower bounds for various $\ell_p$ geometries for
$p\geq1$ including the $\ell_{\infty}$-norm.  Let us setup the
problem.  For any $p\geq1$ and $d\in\mathbb{N}$, let
$\calB_p^{d}\left(a\right)=\{\bx\in\bbR^d : \|\bx\|_p\leq a\}$ denote
the $p$-norm ball with radius $a$ centered at the origin in
$\bbR^{d}$,\footnote{Assuming that the ball is centered at origin is
  without loss of generaility; otherwise, we can translate the ball to
  origin and work with that.} where
$\|\bx\|_{p}=\left(\sum_{j=1}^{d}|\bx_{j}|^{p}\right)^{1/p}$.  Each
client $i\in[n]$ has an input vector $\bx_i\in\mathcal{B}_p^{d}(a)$ and
the server wants to estimate the mean $\barx :=
\frac{1}{n}\sum_{i=1}^{n}\bx_i$. We have two constraints: {\sf (i)}
each client has a communication budget of $b$ bits to transmit the
information about its input vector to the server, and {\sf (ii)} each
client wants to keep its input vector private from the server.  We
develop {\em private-quantization} mechanisms to simultaneously
address these constraints.  Specifically, we design mechanisms
$\calM_i:\calB_p^d(a)\to\{0,1\}^d$ for $i\in[n]$ that are quantized in
the sense that they produce a $b$-bit output and are also locally
differentially private. In other words, $\calM_i$ is $(\eps_0,b)$-LDP
for some $\eps_0\geq0$ (see Definition~\ref{defn:local-DP}).

The procedure goes as follows. client $i\in[n]$ applies a private-quantization mechanism $\calM_i$ on her input $\bx_i$ and obtains a private output $\by_i=\calM_i(\bx_i)$ and sends it to the server. Upon receiving $\by^n=[\by_1,\hdots,\by_n]$, server applies a decoding function to estimate the mean vector $\barx=\frac{1}{n}\sum_{i=1}^n\bx_i$.
Our objective is to design private-quantization mechanisms $\calM_i:\calB_p^d(a)\to\{0,1\}^d$ for all $i\in[n]$ and also a (stochastic) decoding function $\hatx: \(\{0,1\}^{b}\)^{n}\to\calB_p^{d}$ that minimizes the worst-case expected error $\sup_{\{\bx_i\}\in\calB_p^d}\bbE\|\barx-\hatx(\by^n)\|^2$. In other words, we are interested in characterizing the following quantity.
\begin{equation}~\label{eqn1_1}
\begin{aligned}
r_{\eps,b,n}^{p,d}(a)&=\inf_{\{\calM_i\in\calQ_{(\eps,b)}\}} \inf_{\hatx}\sup_{\{\bx_i\}\in\calB_p^{d}(a)}\bbE\left\|\barx-\hatx(\by^{n})\right\|_{2}^{2},
\end{aligned}
\end{equation} 
where $\calQ_{(\eps,b)}$ is the set of all $(\eps, b)$-LDP mechanisms, and
the expectation is taken over the randomness of $\{\calM_i:i\in[n]\}$ and the estimator $\hatx$. 
Note that in~\eqref{eqn1_1} we do not assume any probabilistic assumptions on the vectors $\bx_1,\hdots,\bx_n$. \\

Now we extend the formulation in \eqref{eqn1_1} to a probabilistic model. 
Let $\calP_p^d(a)$ denote the set of all probability density functions on $\calB^{d}_p(a)$. 
For every distribution $\bq\in\calP_p^d(a)$, let $\bmu_{\bq}$ denote its mean.
Since the support of each distribution $\bq\in\calP_p^d$ is $\calB_{p}^{d}(a)$ and $\ell_p$ is a norm, we have that $\bmu_{\bq}\in\calB_{p}^{d}(a)$. For a given unknown distribution $\bq\in\calP_p^d(a)$, client $i\in[n]$ observes $\bx_i$, where $\bx_1,\ldots,\bx_n$ are i.i.d.~according to $\bq$, and the goal for the server is to estimate $\bmu_{\bq}$, while satisfying the same two constraints as above, i.e., only $b$ bits of communication is allowed from any client to the server while preserving the privacy of clients' inputs.
Analogous to \eqref{eqn1_1}, we are interested in characterizing the following quantity.
\begin{equation}~\label{eqn1_2}
\begin{aligned}
R_{\eps,b,n}^{p,d}(a)&=\inf_{\{\calM_i\in\calQ_{(\eps,b)}\}} \inf_{\hatx}\sup_{\bq\in\calP_p^d(a)}\bbE\left\|\bmu_{\bq}-\hatx(\by^{n})\right\|_{2}^{2},
\end{aligned}
\end{equation}  
where the expectation is taken over the randomness of the output $\by^{n}$ and the estimator $\hatx$.

In this paper, we design private-quantization mechanisms $\{\calM_1,\hdots,\calM_n\}$ such that they are symmetric (i.e., $\calM_i$'s are same for all $i\in[n]$) and any client uses only private source of randomness that is not accessible by any other party in the system.

\section{Main Results}\label{sec:MainRes}

This section is divided into two parts.  In
Section~\ref{subsec:optimization-results}, we setup the problem,
describe our algorithm, state our main results for optimization,
including the results on convergence, privacy, and communication bits
used. We also discuss their implications.  One of the main
ingredients in obtaining these results is the compressed \& private
mean estimation, which we study in a variety of settings; the
corresponding results are presented in
Section~\ref{subsec:mean-est-results}. A summary of the notation used throughout the paper is given in Table~\ref{Tab_summary}.

\begin{table}[t!]
\centering
\begin{tabular}{ |c  c |   }
\hline
  Symbol & Description \\ 
 \hline\hline
 $[n]$ & $=\{1,2,\hdots,n\}$, for any $n\in\mathbb{N}$ \\
 $m$ & Total number of clients in the system \\ 
 $k$ & ($\leq m$) Number of clients chosen per iteration \\ 
 $r$ & Total number of samples per client\\
 $s$ & ($\leq r$) Number of samples chosen per client per iteration\\
 $n$ & ($=mr$) Total number of samples in the dataset \\ 
 $q_1$ & ($=\frac{k}{m}$) Can be seen as probability of choosing a client in any iteration\\  
 $q_2$ & ($=\frac{s}{r}$) Can be seen as probability of picking a sample (from a chosen client) in any iteration\\  
 $q$ & ($=q_1q_2=\frac{ks}{mr}$) Can be seen as probability of choosing a sample in any iteration\\  
 $\calD_i$ & Local dataset of client $i$ for $i\in\left[m\right]$ \\ 
 $\calD$& ($\bigcup_{i=1}^{m}\calD_i$) The entire dataset\\
 $\epsilon_0$ & Local differential privacy parameter \\ 
 $\epsilon $ & Central differential privacy parameter \\ 
 $\theta$& ($\in\mathbb{R}^{d}$) Model parameter vector \\ 
 $\calC$ &($\subset\mathbb{R}^{d}$) convex set of interest \\
 $D$ & ($=\|\calC\|_2$) Diameter of the set $\calC$ \\ 
 $L$& Lipschitz continuous parameter\\
 $\calB_{p}^{d}\left(a\right)$ & $\ell_p$ norm ball of radius $a$  \\
 \hline  
\end{tabular}
\caption{Notation used throughout the paper}\label{Tab_summary}
\end{table}

\subsection{Optimization}\label{subsec:optimization-results}
In the subsection, we present a compressed and differentially-private stochastic gradient descent algorithm for the federated learning problem and state our main results about its privacy, communication, and convergence.
The problem that we study is as follows. There are $m$ clients, each client $i\in[m]$ has a local dataset $\calD_i=\{d_{i1},\hdots,d_{ir}\}\in\frS^r$ consisting of $r$ data points. Let $F_i(\theta)$ denotes the local loss function induced by $\calD_i$ at client $i$ evaluated at the model parameter vector $\theta\in\bbR^d$, where $F_i(\theta)=\frac{1}{r}\sum_{j=1}^r f(\theta;d_{ij})$, where $f(\theta; \cdot):\calC\to\bbR$ is a convex function. The goal of the server is to find an optimal model parameter vector $\theta^*\in\calC$ that minimizes $\min_{\theta\in\calC}\(F(\theta)=\frac{1}{m}\sum_{i=1}^mF_i(\theta)\)$; also see \eqref{eq:EMP}, while satisfying the privacy constraint of a single data point at any client, as formalized in Section~\ref{sec:prelims}, and also the communication constraints.

In our proposed compressed and differentially-private SGD algorithm, 
at each step, we choose at random a set $\mathcal{U}_t$ of $k\leq m$ clients out of $m$ clients. Each client $i\in\mathcal{U}_t$ computes the gradient $\nabla_{\theta_t} f\left(\theta_{t};d_{ij}\right)$ for a random subset $\mathcal{S}_{it}$ of $s\leq r$ samples. The $i$'th client clips the $\ell_p$-norm of the gradient $\nabla_{\theta_t} f\left(\theta_t;d_{ij}\right)$ for each $j\in\mathcal{S}_{it}$ and applies the LDP-compression mechanism $\mathcal{R}_p$ (with the privacy parameter $\epsilon_0$) to the clipped gradients. After that, each client $i$ sends the set of $s$ LDP-compressed gradients $\lbrace \mathcal{R}_p\left(\mathbf{g}_{t}\left(d_{ij}\right)\right)\rbrace_{j\in\mathcal{S}_{it}}$ in a communication-efficient manner to the secure shuffler. The shuffler randomly shuffles (i.e., outputs a random permutation of) all the received $ks$ gradients and sends them to the server. Finally, the server takes the average of the received gradients and updates the parameter vector. 
We describe this procedure in Algorithm~\ref{algo:optimization-algo}. 
Let $\ell_g$ denote the dual norm of $\ell_p$ norm, where $\frac{1}{p}+\frac{1}{g}=1$ and $p,g\geq1$. Thus, when the loss function $f\left(\theta,d_{ij}\right)$ is convex and $L$-Lipschitz continuous with respect to $\ell_g$-norm, then the gradient $\nabla_{\theta} f\left(\theta;.\right)$ has a bounded $\ell_p$ norm~\cite[Lemma~$2.6$]{shalev2012online}. In this case, we do not need the clipping step.

In the next theorems, we state the privacy guarantees, the communication cost per client, and the privacy-convergence trade-offs for the CLDP-SGD Algorithm. Let $n=mr$ denote the total number of data points in the dataset $\mathcal{D}$. Observe that the probability that an arbitrary data point $d_{ij}\in\mathcal{D}$ is chosen at time $t\in\left[T\right]$ is given by $q=\frac{ks}{mr}$.

\begin{algorithm}[t]
\caption{$\mathcal{A}_{\text{cldp}}$: CLDP-SGD}\label{algo:optimization-algo}
\begin{algorithmic}[1]
\State \textbf{Inputs:} Datasets $\mathcal{D}=\bigcup_{i\in\left[m\right]}\mathcal{D}_i$, $\mathcal{D}_i=\lbrace d_{i1},\ldots,d_{ir}\rbrace$, loss function $F\left(\theta\right)=\frac{1}{m r}\sum_{i=1}^{m}\sum_{j=1}^{r} f\left(\theta;d_{ij}\right)$, LDP privacy parameter $\epsilon_0$, gradient norm bound $C$, and learning rate $\eta_t$.
\State \textbf{Initialize:} $\theta_0\in\mathcal{C}$ 
\For {$t\in\left[T\right]$}
\State \textbf{Sampling of clients:} A random  set $\mathcal{U}_t$ of $k$ clients is chosen.
\For  {clients $i\in\mathcal{U}_t$}
\State \textbf{Sampling of data:} Client $i$ chooses uniformly at random a set $\mathcal{S}_{it}$ of $s$ samples.
\For {Samples $j\in\mathcal{S}_{it}$}
\State \textbf{Compute gradient:} $\mathbf{g}_t\left(d_{ij}\right)\gets \nabla_{\theta_t}f\left(\theta_t;d_{ij}\right)$
\State \textbf{Clip gradient:} $\tilde{\mathbf{g}}_t\left(d_{ij}\right)\gets \mathbf{g}_t\left(d_{ij}\right)/\max\left\{1,\frac{\|\mathbf{g}_t\left(d_{ij}\right)\|_p}{C}\right\}$\footnotemark
\State \textbf{LDP-compressed gradient:} $\mathbf{q}_t\left(d_{ij}\right)\gets \mathcal{R}_p\left(\tilde{\mathbf{g}}_t\left(d_{ij}\right)\right)$
\EndFor
\State Client $i$ sends the set of private-compressed gradients $\lbrace \mathbf{q}_{t}\left(d_{ij}\right):j\in\mathcal{S}_{it}\rbrace$ to the shuffler.
\EndFor
\State \textbf{Aggregate:} $\overline{\mathbf{g}}_t\gets \frac{1}{ks}\sum_{i\in\mathcal{U}}\sum_{j\in\mathcal{S}_{it}}\bq_t\left(d_{ij}\right)$
\State \textbf{Gradient Descent} $\theta_{t+1}\gets \prod_{\mathcal{C}}\left( \theta_t -\eta_t \overline{\mathbf{g}}_t\right)$ 
\EndFor
\State \textbf{Output:} The model $\theta_{T}$ and the privacy parameters $\epsilon$, $\delta$
\end{algorithmic}
\end{algorithm}
\footnotetext{Note that gradient clipping may not preserve unbiasedness of the stochastic gradients. However for the case when the loss function $f$ is $L$-Lipschitz, this is not necessary for the following reason.
If the loss function $f$ is $L$-Lipschitz (with respect to the model parameters) in the dual norm $\ell_g$, where $\frac{1}{p}+\frac{1}{g}=1, p,g\geq1$, then  the norm of the gradients (with respect to some $\ell_p$-norm, for $p\geq1$) is bounded, and hence we do not need to clip it.}

\begin{theorem}\label{thm:main-opt-result} 
Let the set $\mathcal{C}$ be convex with diameter $D,$\footnote{Diameter of a bounded set $\calC\subseteq\bbR^d$ is defined as $\sup_{\bx,\by\in\calC}\|\bx-\by\|$.} and the function $f\left(\theta;.\right):\mathcal{C}\to \mathbb{R}$ be convex and $L$-Lipschitz continuous with respect to the $\ell_g$-norm, which is the dual of the $\ell_p$-norm.\footnote{For any data point $d\in\frS$, the function $f:\calC\to\bbR$ is $L$-Lipschitz continuous w.r.t.\ $\ell_g$-norm if for every $\theta_1,\theta_2\in\calC$, we have $|f(\theta_1;d)-f(\theta_2;d)|\leq L\|\theta_1-\theta_2\|_g$.} Let $\theta^{*}=\arg\min_{\theta\in\mathcal{C}} F\left(\theta\right)$ denote the minimizer of the problem~\eqref{eq:EMP}. For $s=1$ and $q=\frac{k}{mr}$, if we run Algorithm $\calA_{cldp}$ with $\eps_0=\calO\(\sqrt{\frac{n\log(2/\delta)}{qT\log(2qT/\delta)}}\)$, then we have
\begin{enumerate}
\item \textbf{Privacy:} $\calA_{cldp}$ is $\left(\epsilon,\delta\right)$-DP, where $\delta>0$ is arbitrary, and 
\begin{equation}\label{final-epsilon}
\epsilon=\mathcal{O}\left(\epsilon_0\sqrt{\frac{qT\log\left(2qT/\delta\right)\log\left(2/\delta\right)}{n}}\right).
\end{equation}

\item \textbf{Communication:} $\calA_{cldp}$ requires $\frac{k}{m}s\times\left(\log\left(e\right)+\log\left(\frac{s+2^{b}-1}{s}\right)\right)$ bits of communication in expectation\footnote{A client communicates in an iteration only when that client is selected (sampled) in that iteration.} per client per iteration, where expectation is taken with respect to the sampling of clients. Here, $b=\log\left(d\right)+1$ if $p\in\lbrace 1,\infty\rbrace$ and $b=d\left(\log\left(e\right)+1\right)$ otherwise. 
\item \textbf{Convergence:} For $G^2=L^2\max\lbrace d^{1-\frac{2}{p}},1\rbrace \( 1+ \frac{cd}{qn}\left(\frac{e^{\epsilon_0}+1}{e^{\epsilon_0}-1}\right)^2 \)$, if we run $\calA_{cldp}$ with learning rate schedule $\eta_t=\frac{D}{G\sqrt{t}}$, then 
we have
\begin{equation}\label{general-convergence}
\mathbb{E}\left[F\left(\theta_{T}\right)\right] - F\left(\theta^{*}\right) \leq \calO\(\frac{LD\log(T)\max\lbrace d^{\frac{1}{2}-\frac{1}{p}},1\rbrace}{\sqrt{T}} \sqrt{\frac{cd}{qn}}\left(\frac{e^{\epsilon_0}+1}{e^{\epsilon_0}-1}\right) \).
\end{equation}
where $c$ is an absolute constant, see Lemma~\ref{lem:2nd-moment-bound} on page~\pageref{lem:2nd-moment-bound}.
\end{enumerate}
\end{theorem}
We prove Theorem~\ref{thm:main-opt-result} in Section~\ref{sec:OptPerf}.
\begin{remark}[Arbitrary SGD mini-batch size $s$]
The communication and convergence results in Theorem 1 are general and hold for any $s\in\left[r\right]$, but the privacy result is stated when $s = 1$, i.e., each client only samples a single data point in each SGD iteration. We also have results for any mini-batch size $s\in\left[r\right]$ (see Appendix~\ref{app:supp_composition_priv}), and the desired privacy amplification occurs when $k=m$, i.e., when we select all the clients in every iteration, but for other cases it is unclear if we get the same result (see Appendix~\ref{app:supp_composition_priv} for more details) .
%
%
\end{remark}
%
%
%
\begin{remark}[Recovering the Result {\cite[ESA]{erlingsson2020encode}}]\label{rem:CompESA}
In \cite{erlingsson2020encode}, each client has only one data point and all clients participate in each iteration, and gradients have bounded $\ell_2$-norm.
If we put $p=2$, $T= n/\log^2(n)$, and $q=1$ in \eqref{general-convergence}, we get the following privacy-accuracy trade-off, which is the same as that in \cite[Theorem VI.1]{erlingsson2020encode}.
\begin{align*}
\mathbb{E}\left[F\left(\theta_{T}\right)\right] - F\left(\theta^{*}\right) \leq \calO\(\frac{LD\log^2(n)\sqrt{d}}{n} \left(\frac{e^{\epsilon_0}+1}{e^{\epsilon_0}-1}\right) \); \quad  
\epsilon = \mathcal{O}\left(\epsilon_0\sqrt{\frac{T\log\left(T/\delta\right)\log\left(1/\delta\right)}{n}}\right) 
\end{align*}
We want to emphasize that the above privacy-accuracy trade-off in \cite{erlingsson2020encode} is achieved by full-precision gradient exchange, whereas,
we can achieve the same trade-off with compressed gradients. 
Moreover, our results are in more general setting, where clients' local datasets have multiple data-points (no bound on that) and we do two types of sampling, one of clients and other of data for SGD.
\end{remark}
\begin{remark}[Optimality of CLDP-SGD for $\ell_2$-norm case]\label{remark:optimality-algo-L-2}
Suppose $\epsilon=\calO\(\log(1/\delta)\)$, which includes $\eps=\calO(1)$, as $\delta \ll 1$.
Substituting $\epsilon_0=\epsilon\sqrt{\frac{n}{qT\log\left(2qT/\delta\right)\log\left(2/\delta\right)}}$, 
$T=n/q$, and $p=2$ in \eqref{general-convergence}, we get
\begin{equation}
\mathbb{E}\left[ F\left(\theta_{T}\right)\right]-F\left(\theta^{*}\right)=\mathcal{O}\left(\frac{LD\log^{\frac{3}{2}}\left(\frac{n}{\delta}\right)\sqrt{d \log\left(\frac{1}{\delta}\right)}}{n\epsilon}\right). \label{optimal-accuracy-L2}
\end{equation}
This matches the optimal excess risk of central differential privacy presented in~\cite{bassily2014private}.
Note that the results in \cite{bassily2014private} are for centralized SGD with full precision gradients, whereas, our results are for federated learning (which is a distributed setup) with compressed gradient exchange.
\end{remark}

\subsection{Compressed and Private Mean Estimation}\label{subsec:mean-est-results}
In this subsection, we state our lower and upper bound results on minimax risks both in the worst case model (see \eqref{eqn1_1}) and the probabilistic model (see \eqref{eqn1_2}). 
For the lower bounds, we state our results when there is no communication constraints, and for clarity, we denote the corresponding minimax risks by $r_{\epsilon,\infty,n}^{p,d}(a)$ and $R_{\epsilon,\infty,n}^{p,d}(a)$.

\begin{theorem} 
\label{thm:minimax_prob-lb} 
For any $d,n\geq 1$, $a,\eps_0>0$, and $p\in\left[1,\infty\right]$, 
the minimax risk in~\eqref{eqn1_2} satisfies
\[ 
R_{\epsilon,\infty,n}^{p,d}(a) \geq \begin{cases}
\Omega\left(a^2\min\left\{ 1,\frac{d}{n\eps_0^{2}}\right\}\right)& \text{if}\ 1\leq p\leq 2,\\
\Omega\left(a^2d^{1-\frac{2}{p}}\min\left\{ 1,\frac{d}{n\min\{\eps_0,\eps_0^2\}}\right\}\right)& \text{if}\ p\geq 2.
\end{cases}
\]
\end{theorem}

\begin{theorem} 
\label{thm:minimax_worst-lb} 
For any $d,n\geq 1$, $a,\eps_0>0$, and $p\in\left[1,\infty\right]$, the minimax risk in~\eqref{eqn1_1} satisfies
\[
r_{\epsilon,\infty,n}^{p,d}(a) \geq \begin{cases}
\Omega\left(a^2\min\left\{ 1,\frac{d}{n\eps_0^{2}}\right\}\right)& \text{if}\ 1\leq p\leq 2, \\
\Omega\left(a^2d^{1-\frac{2}{p}}\min\left\{ 1,\frac{d}{n\min\{\eps_0,\eps_0^2\}}\right\}\right)& \text{if}\ p\geq 2.
\end{cases}
\]
\end{theorem}

We prove Theorem~\ref{thm:minimax_prob-lb} and Theorem~\ref{thm:minimax_worst-lb} in Section~\ref{subsec:minimax_prob-lb} and Section~\ref{subsec:minimax_worst-lb}, respectively.

\begin{theorem}~\label{thm:minimax_comm_lb}
For any private-randomness, symmetric mechanism $\mathcal{R}$ with communication budget $b<\log\left(d\right)$ bits per client, and any decoding function $g:\lbrace 0,1\rbrace^{b}\to \mathbb{R}^{d}$, when $\hatx=\frac{1}{n}\sum_{i=1}^{n}g\left(\mathcal{R}\left(\bx_i\right)\right)$, we have
\begin{equation}
r_{\epsilon,b,n}^{p,d}(a) > a^2\max\left\{1,d^{1-\frac{2}{p}}\right\}.
\end{equation}
\end{theorem}
\begin{remark}
Note that Theorem~\ref{thm:minimax_comm_lb} works only when the estimator $\hatx$ applies the decoding function $g$ on individual responses and then takes the average.  We leave its extension for arbitrary decoders as a future work.
\end{remark}
We prove Theorem~\ref{thm:minimax_comm_lb} in Section~\ref{subsec:minimax_comm_lb_proof}.

Though our lower bound results are for arbitrary estimators $\hatx(\by^n)$, for the minimax risk estimation problems~\eqref{eqn1_1} and~\eqref{eqn1_2}, we can show that the optimal estimator $\hatx(\by^n)$ is a {\em deterministic} function of $\by^n$. In other words, the randomized decoder does not help  in reducing the minimax risk. See Lemma~\ref{lem:deterministic-estimator} in Appendix~\ref{app:minimax}.

\begin{theorem}[$\ell_1$-norm]\label{thm:L1-norm_ub}
For any $d,n\geq 1$, $a,\eps_0>0$, we have
\begin{align*}
r_{\eps_0,b,n}^{1,d}\left(a\right) \leq \frac{a^2d}{n}\left(\frac{e^{\eps_0}+1}{e^{\eps_0}-1}\right)^{2} \quad \text{and} \quad
R_{\eps_0,b,n}^{1,d}\left(a\right) \leq \frac{4a^2d}{n}\left(\frac{e^{\eps_0}+1}{e^{\eps_0}-1}\right)^{2},
\end{align*}
for $b=\log(d)+1$.
\end{theorem}

\begin{theorem}[$\ell_2$-norm]\label{thm:L2-norm_ub}
For any $d,n\geq 1$, $a,\eps_0>0$, we have
\begin{align*}
r_{\eps_0,b,n}^{2,d}\left(a\right)\leq  \frac{6a^2 d}{n}\left(\frac{e^{\eps_0}+1}{e^{\eps_0}-1}\right)^{2} \quad \text{and} \quad
R_{\eps_0,b,n}^{2,d}\left(a\right) \leq  \frac{14 a^2 d}{n}\left(\frac{e^{\eps_0}+1}{e^{\eps_0}-1}\right)^{2},
\end{align*}
for $b=d\log(e)+1$.
\end{theorem}

\begin{theorem}[$\ell_{\infty}$-norm]\label{thm:L-infty-norm_ub}
For any $d,n\geq 1$, $a,\eps_0>0$, we have
\begin{align*}
r_{\eps_0,b,n}^{\infty,d}(a) \leq \frac{a^2 d^2}{n}\left(\frac{e^{\eps_0}+1}{e^{\eps_0}-1}\right)^{2} \quad \text{and} \quad
R_{\eps_0,b,n}^{\infty,d}(a) \leq \frac{4a^2d^2}{n}\left(\frac{e^{\eps_0}+1}{e^{\eps_0}-1}\right)^{2},
\end{align*}
for $b=\log(d)+1$.
\end{theorem}
We prove Theorem~\ref{thm:L1-norm_ub}, Theorem~\ref{thm:L2-norm_ub}, and Theorem~\ref{thm:L-infty-norm_ub}, in Section~\ref{subsec:L1-norm_ub_proof}, Section~\ref{subsec:L2-norm_ub_proof}, and Section~\ref{subsec:L-infty-norm_ub_proof}, respectively.

Note that when $\eps_0=\calO(1)$, then the upper and lower bounds on minimax risks match for $p\in[1,2]$.
Furthermore, when $\eps_0\leq1$, then they match for all $p\in[1,\infty]$.

Now we give a general achievability result for any $\ell_p$-norm ball $\calB_p^d(a)$ for any $p\in[1,\infty)$.
For this, we use standard inequalities between different norms, and probabilistically use the mechanisms for $\ell_1$-norm or $\ell_2$-norm with expanded radius of the corresponding ball. 
We assume that every work can pick any mechanisms with the same probability $\bar{p}\in[0,1]$.
This gives the following result, which we prove in Section~\ref{subsec:L-p-norm_ub_proof}.
\begin{corollary}[General $\ell_p$-norm, $p\in[1,\infty)$]\label{corol:L-p-norm_ub}
Suppose clients pick the mechanism for $\ell_1$-norm with probability $\bar{p}\in[0,1]$. Then, 
for any $d,n\geq 1$, $a,\eps_0>0$, we have:
\begin{align}
r_{\eps_0,b,n}^{p,d}(a) \ &\leq \ \bar{p}\,d^{2-\frac{2}{p}}\cdot r_{\eps_0,b,n}^{1,d}(a) + (1-\bar{p})\max\big\lbrace d^{1-\frac{2}{p}}, 1\big\rbrace\cdot r_{\eps_0,b,n}^{2,d}(a), \label{achieve_L-p-norm_worst} \\
R_{\eps_0,b,n}^{p,d}(a) \ &\leq \ \bar{p}\,d^{2-\frac{2}{p}}\cdot R_{\eps_0,b,n}^{1,d}(a) + (1-\bar{p})\max\big\lbrace d^{1-\frac{2}{p}}, 1\big\rbrace\cdot R_{\eps_0,b,n}^{2,d}(a). \label{achieve_L-p-norm_prob}
\end{align}
for $b=\bar{p}\log(d)+(1-\bar{p})d\log(e)+1$. Note that this communication is in expectation, which is taken over the sampling of selecting $\ell_1$ or $\ell_2$ mechanisms.
\end{corollary}
We can recover Theorem~\ref{thm:L1-norm_ub} by setting $p=1$ and $\bar{p}=1$ and Theorem~\ref{thm:L2-norm_ub} by setting $p=2$ and $\bar{p}=0$.

\section{Compressed and Private Mean Estimation}
\label{sec:PrivMeanEst}

In this section, we study the private mean-estimation problem in the
minimax framework given in Section~\ref{subsec:cp-mean-est}. Note that
in this section we focus on giving $(\eps_{0},b)$-CLDP)
privacy-communication guarantees for the mean-estimation problem and
give the performance of schemes in terms of the associated minimax
risk. This framework is applied at each round of the optimization
problem, and is then converted to the eventual central DP privacy
guarantees using the shuffling framework in Section~
\ref{sec:OptPerf}, yielding the main result Theorem
\ref{thm:main-opt-result} stated in Section~\ref{sec:MainRes}.

This section is divided into six subsections.
We prove the lower bound results (Theorems~\ref{thm:minimax_prob-lb}, \ref{thm:minimax_worst-lb}) in the first two subsections and the achievable results (Theorems~\ref{thm:L1-norm_ub}, \ref{thm:L2-norm_ub}, \ref{thm:L-infty-norm_ub}, and Corollary~\ref{corol:L-p-norm_ub}) in the last four subsections, respectively.

We prove lower bounds for private mechanisms with no communication constraints, and for clarity, we denote such mechanisms by $(\eps,\infty)$-CLDP mechanisms. Our achievable schemes use finite amount of randomness.

For lower bounds, for simplicity, we assume that the inputs come from an $\ell_p$-norm ball of unit radius -- the bounds will be scaled by the factor of $a^2$ if inputs come from an $\ell_p$-norm ball of radius $a$. For convenience, we denote $\calB_p^d(1),\calP_p^d(1),r_{\eps,b,n}^{p,d}(1)$, and $R_{\eps,b,n}^{p,d}(1)$ by $\calB_p^d ,\calP_p^d, r_{\eps,b,n}^{p,d}$, and $R_{\eps,b,n}^{p,d}$, respectively.

\subsection{Lower Bound on $R_{\eps,\infty,n}^{p,d}$: Proof of Theorem~\ref{thm:minimax_prob-lb}}\label{subsec:minimax_prob-lb}
Theorem~\ref{thm:minimax_prob-lb} states separate lower bounds on $R_{\eps,\infty,n}^{p,d}$ depending on whether $p\geq 2$ or $p\leq2$ (at $p=2$, both bounds coincide), and we prove them below in Section~\ref{subsubsec:minimax_prob-lb_p-geq-2} and Section~\ref{subsubsec:minimax_prob-lb_p-leq-2}, respectively. \\
\subsubsection{Lower bound for $p\in[2,\infty]$}\label{subsubsec:minimax_prob-lb_p-geq-2} 
The main idea of the lower bound is to transform the problem to the private mean estimation when the inputs are sampled from Bernoulli distributions.
Recall that $\calP_p^d$ denote the set of all distributions on the $p$-norm ball $\calB^{d}_p$.
Let $\calP^{\text{Bern}}_{p,d}$ denote the set of Bernoulli distributions on $\left\{ 0,\frac{1}{d^{\nicefrac{1}{p}}}\right\}^{d}$, 
i.e., any element of $\calP^{\text{Bern}}_{p,d}$ is a product of $d$ independent Bernoulli distributions, one for each coordinate.
We first prove a lower bound on $R_{\eps,\infty,n}^{p,d}$ when the input distribution belongs to $\calP^{\text{Bern}}_{p,d}$.
\begin{lemma}\label{lem:bernoulli}
For any $p\in[2,\infty]$, we have
\begin{equation}\label{minimax_prob-lb_p-geq-2_interim1}
\inf_{\{\calM_i\}\in\calQ_{(\eps,\infty)}}\inf_{\hatx}\sup_{\bq\in\calP^{\text{Bern}}_{p,d}}\bbE\left\|\bmu_{\bq}-\hatx\left(\by^n\right)\right\|_{2}^{2} \geq \Omega\(d^{1-\frac{2}{p}}\min\left\lbrace 1,\frac{d}{n\min\{\eps,\eps^2\}}\right\rbrace\).
\end{equation}
\end{lemma}
\begin{proof}
The proof is straightforward from the proof of Duchi and Rogers~\cite[Corollary $3$]{duchi2019lower}.
In their setting, $\calP^{\text{Bern}}_{p,d}$ is supported on $\{0,1\}^d$, and they proved a lower bound of $\Omega\(\min\left\lbrace 1,\frac{d}{n\min\{\eps,\eps^2\}}\right\rbrace\)$. In our setting, since $\calP^{\text{Bern}}_{p,d}$ is supported on $\left\{ 0,\frac{1}{d^{\nicefrac{1}{p}}}\right\}^{d}$, we can simply scale the elements in the support of $\calP^{\text{Bern}}_{p,d}$ by a factor of $1/d^{\nicefrac{1}{p}}$, which will also scale the mean $\bmu_{\bq}$ by the same factor. 
Note that the best estimator $\hatx$ will be equal to the scaled version of the best estimator from \cite[Corollary $3$]{duchi2019lower} with the same value $1/d^{\nicefrac{1}{p}}$. This proves Lemma~\ref{lem:bernoulli}.
\end{proof}

In order to use Lemma~\ref{lem:bernoulli}, first observe that for every $\bx\in\calP^{\text{Bern}}_{p,d}$, we have $\|\bx\|_p\leq 1$, which implies that $\bx\in\calP_p^d$. Thus we have $\calP^{\text{Bern}}_{p,d}\subset\calP_p^d$.
Now our bound on $R_{\eps,\infty,n}^{p,d}$ trivially follows from the following inequalities:
\begin{align}
R_{\eps,\infty,n}^{p,d} = \inf_{\{\calM_i\}\in\calQ_{(\eps,\infty)}}\inf_{\hatx}\sup_{\bq\in\calP_p^d}\bbE\left\|\bmu_{\bq}-\hatx\left(\by^n\right)\right\|_{2}^{2} &\geq \inf_{\{\calM_i\}\in\calQ_{(\eps,\infty)}}\inf_{\hatx}\sup_{\bq\in\calP^{\text{Bern}}_{p,d}}\bbE\left\|\bmu_{\bq}-\hatx\left(\by^n\right)\right\|_{2}^{2} \notag \\
&\geq \Omega\(d^{1-\frac{2}{p}}\min\left\lbrace 1,\frac{d}{n\min\{\eps,\eps^2\}}\right\rbrace\), \label{minimax_prob-lb_p-geq-2_interim2}
\end{align}

where the last inequality follows from \eqref{minimax_prob-lb_p-geq-2_interim1}. \\

\subsubsection{Lower bound for $p\in[1,2]$}\label{subsubsec:minimax_prob-lb_p-leq-2}
Fix an arbitrary $p\in[1,2]$.
Note that $\|\bx\|_p\leq \|\bx\|_1$, which implies that $\calB_1^{d}\subset\calB_p^{d}$, and therefore, 
we have $\calP_1^d\subset\calP_p^d$. These imply that the lower bound derived for $\calP_1^d$ also holds for $\calP_p^d$, i.e., $R_{\eps,\infty,n}^{p,d}\geq R_{\eps,\infty,n}^{1,d}$ holds for any $p\in[1,2]$. 
So, in the following, we only lower-bound $R_{\eps,\infty,n}^{1,d}$.

The main idea of the lower bound is to transform the problem to the
private discrete distribution estimation when the inputs are sampled
from a discrete distribution taken from a simplex in $d$ dimensions.
Recall that $\calP_1^d$ denotes all probability density functions $q$
over the $1$-norm ball $\calB_1^d$.  Note that $q$ may be a continuous
distribution supported over all of $\calB_1^d$.  Let $\hatcalP_1^d$
denote a set of all discrete distributions $\bq$ supported over the
$d$ standard basis vectors $\be_1,\hdots,\be_d$, \emph{i.e.,} the
distribution has support on $\{\be_1,\hdots,\be_d$\}. Since
$\{\be_1,\hdots,\be_d\}\subset \calB_1^d$, we have
$\hatcalP_1^d\subset\calP_1^d$. Moreover, since any $q\in\hatcalP_1^d$
is a discrete distribution, by abusing notation, we describe $q$
through a $d-$dimensional vector $\bq$ of its probability mass
function.  Note that, for any $\bq\in\hatcalP_1^d$, the average over
this distribution is $\bmu_{\bq}=\mathbb{E}_{\bq}[\mathbf{U}]$, where
$\mathbb{E}_{\bq}[\cdot]$ denotes the expectation over the
distribution $\bq$ for a discrete random variable $\mathbf{U}\sim\bq$,
where we denote $q_i=\Pr[\mathbf{U}=\be_i]$. Therefore we have
$\bmu_{\bq}=\sum_{i=1}^dq_i\be_i = (q_1,\hdots,q_d)^T = \bq$, for
every $\bq\in\hatcalP_1^d$. Let $\Delta_d$ denote the probability
simplex in $d$ dimensions. Since the discrete distribution
$q\in\hatcalP_1^d$ is representable as $\bq\in\Delta_d$, we have an
isomorphism between $\Delta_d$ and $\hatcalP_1^d$, \emph{i.e.,} we can
equivalently think of $\hatcalP_1^d=\Delta_d$. Fix arbitrary
$(\eps,\infty)$-CLDP mechanisms $\{\calM_i:i\in[n]\}$ and an estimator
$\hatx$.  Using the above notations and observations, we have:
\begin{equation}\label{minimax_prob-lb_p-leq-2_interim1}
\sup_{\bq\in\calP_1^d}\bbE\left\|\bmu_{\bq}-\hatx\left(\by^n\right)\right\|_{2}^{2} \geq \sup_{\bq\in\hatcalP_1^d}\bbE\left\|\bmu_{\bq}-\hatx\left(\by^n\right)\right\|_{2}^{2} 
= \sup_{\bq\in\hatcalP_1^d}\bbE\left\|\bq-\hatx\left(\by^n\right)\right\|_{2}^{2}.
\end{equation}
Using $\hatcalP_1^d=\Delta_d$, and taking the infimum in \eqref{minimax_prob-lb_p-leq-2_interim1} over all $(\eps,\infty)$-CLDP mechanisms $\{\calM_i:i\in[n]\}$ and estimators $\hatx$, we get
\begin{equation}\label{minimax_prob-lb_p-leq-2_interim2}
\inf_{\lbrace \calM_i\in\calQ_{\left(\eps,\infty\right)}\rbrace}\inf_{\hatx}\sup_{\bq\in\calP_1^d}\bbE\left\|\bmu_{\bq}-\hatx\left(\by^n\right)\right\|_{2}^{2} \geq 
\inf_{\lbrace \calM_i\in\calQ_{\left(\eps,\infty\right)}\rbrace}\inf_{\hatx}\sup_{\bq\in\Delta_d}\bbE\left\|\bq-\hatx\left(\by^n\right)\right\|_{2}^{2}.
\end{equation}
Girgis et al.~\cite[Theorem $1$]{Girgis_SuccRefPriv20} lower-bounded the RHS of \eqref{minimax_prob-lb_p-leq-2_interim2} in the context of characterizing a privacy-utility-randomness tradeoff in LDP. When specializing to our setting, where we are not concerned about the amount of randomness used, their lower bound result gives 
$\inf_{\lbrace \calM_i\in\calQ_{\left(\eps,\infty\right)}\rbrace}\inf_{\hatx}\sup_{\bq\in\Delta_d}\bbE\left\|\bq-\hatx\left(\by^n\right)\right\|_{2}^{2} \geq \Omega\(\min\left\lbrace 1,\frac{d}{n\eps^{2}}\right\rbrace\)$.
Substituting this in \eqref{minimax_prob-lb_p-leq-2_interim2} gives
\begin{equation}\label{minimax_prob-lb_p-leq-2_interim3}
R_{\eps,\infty,n}^{1,d} = \inf_{\lbrace \calM_i\in\calQ_{\left(\eps,\infty\right)}\rbrace}\inf_{\hatx}\sup_{\bq\in\calP_1^d}\bbE\left\|\bmu_{\bq}-\hatx\left(\by^n\right)\right\|_{2}^{2} \geq \Omega\(\min\left\lbrace 1,\frac{d}{n\eps^{2}}\right\rbrace\).
\end{equation}

\subsection{Lower Bound on $r_{\eps,\infty,n}^{p,d}$: Proof of Theorem~\ref{thm:minimax_worst-lb}}\label{subsec:minimax_worst-lb}
Similar to Section~\ref{subsec:minimax_prob-lb}, we prove the lower bound on $r_{\eps,\infty,n}^{p,d}$ separately depending on whether $p\geq2$ or $p\leq2$ (at $p=2$, both bounds coincide) below in Section~\ref{subsubsec:minimax_worst-lb_p-geq-2} and Section~\ref{subsubsec:minimax_worst-lb_p-leq-2}, respectively. 
In both the proofs, the main idea is to transform the worst-case lower bound to the average case lower bound and then use relation between different norms.

\subsubsection{Lower bound for $p\in[2,\infty]$}\label{subsubsec:minimax_worst-lb_p-geq-2}

Fix arbitrary $(\eps,\infty)$-CLDP mechanisms $\{\calM_i:i\in[n]\}$ and an estimator $\hatx$. 
It follows from \eqref{minimax_prob-lb_p-geq-2_interim2} that there exists a distribution $\bq\in\calP_p^d$, such that if we sample $\bx_i^{(q)}\sim\bq$, i.i.d.\ for all $i\in\left[n\right]$ and letting $\by_i=\calM_i(\bx_i^{(q)})$, we would have $\bbE\left\|\bmu_{\bq}-\hatx\left(\by^n\right)\right\|_{2}^{2}\geq \Omega\( d^{1-\frac{2}{p}}\min\left\lbrace 1,\frac{d}{n\min\{\eps,\eps^2\}}\right\rbrace\)$. We have
\begin{align}
\sup_{\{\bx_i\}\in\calB_p^{d}}\bbE\left\|\frac{1}{n}\sum_{i=1}^{n}\bx_i-\hatx\left(\by^n\right)\right\|_{2}^{2} &\stackrel{\text{(a)}}{\geq}
\bbE\left\|\frac{1}{n}\sum_{i=1}^{n}\bx_i^{(q)}-\hatx\left(\by^n\right)\right\|_{2}^{2} \notag \\
&\stackrel{\text{(b)}}{\geq} 
\frac{1}{2}\bbE\left\|\bmu_{\bq}-\hatx\left(\by^n\right)\right\|_{2}^{2} -
\bbE\left\|\frac{1}{n}\sum_{i=1}^{n}\bx_i^{(q)}-\bmu_{\bq}\right\|_{2}^{2} \label{minimax_worst-lb_p-geq-2_interim1} \\
&\stackrel{\text{(c)}}{\geq} \Omega\(d^{1-\frac{2}{p}}\min\left\lbrace 1,\frac{d}{n\min\{\eps,\eps^2\}}\right\rbrace\) - \frac{d^{1-\frac{2}{p}}}{n} \notag \\
&\stackrel{\text{(d)}}{\geq} \Omega\(d^{1-\frac{2}{p}}\min\left\lbrace 1,\frac{d}{n\min\{\eps,\eps^2\}}\right\rbrace\) \label{minimax_worst-lb_p-geq-2_interim2}
\end{align}
In the LHS of (a), the expectation is taken over the randomness of the
mechanisms $\{\calM_i\}$ and the estimator $\hatx$; whereas,
in the RHS of (a), in addition, the expectation is also taken over
sampling $\bx_i$'s from the distribution ${\bq}$. Moreover (a) holds
since the LHS is supremum $\{\bx_i\}\in\calB_p^{d}$ and the RHS of (a)
takes expectation w.r.t.\ a distribution over $\calB_p^{d}$ and hence lower-bounds the
LHS.  The inequality $(b)$ follows from the Jensen's inequality $2\|{\bf
  u}\|_{2}^{2}+2\|{\bf v}\|_{2}^{2}\geq \|\bu + {\bf v}\|_{2}^{2}$
by setting $\bu=\frac{1}{n}\sum_{i=1}^{n}\bx_i^{(q)}-\hatx(\by^n)$ and
${\bf v}=\bmu_{\bq}-\frac{1}{n}\sum_{i=1}^{n}\bx_i^{(q)}$. 
In (c) we used $\bbE\left\|\frac{1}{n}\sum_{i=1}^{n}\bx_i^{(q)}-\bmu_{\bq}\right\|_{2}^{2} \leq \frac{d^{1-\frac{2}{p}}}{n}$, which we show below.
In (d), we assume $\min\{\eps,\eps^2\} \leq \calO(d)$.

Note that for any vector $\bu\in\mathbb{R}^d$, we have $\|\bu\|_{2}\leq d^{\frac{1}{2}-\frac{1}{p}} \|\bu\|_{p}$, for any $p\geq 2$. Since each $\bx_i^{(q)}\in\calB_p^d$, which implies $\|\bx_i^{(q)}\|_p\leq1$, we have that $\|\bx_i^{(q)}\|_2\leq d^{\frac{1}{2}-\frac{1}{p}}$. Hence, $\bbE\|\bx_i^{(q)}\|_{2}^{2} \leq d^{1-\frac{2}{p}}$ holds for all $i\in[n]$. 
Now, since $\bx_i$'s are i.i.d.\ with $\bbE[\bx_i^{(q)}]=\bmu_{\bq}$, we have 
\begin{align}\label{2-norm-bnd}
\bbE\left\|\frac{1}{n}\sum_{i=1}^{n}\bx_i^{(q)}-\bmu_{\bq}\right\|_{2}^{2} = \frac{1}{n^2}\sum_{i=1}^{n}\bbE\left\|\bx_i^{(q)}-\bmu_{\bq}\right\|_{2}^{2} 
\stackrel{\text{(a)}}{\leq} \frac{1}{n^2}\sum_{i=1}^{n}\bbE\left\|\bx_i^{(q)}\right\|_{2}^{2} 
\leq \frac{1}{n^2}\sum_{i=1}^{n}d^{1-\frac{2}{p}}
= \frac{d^{1-\frac{2}{p}}}{n},
\end{align}
where (a) uses $\bbE\|\bx-\bbE[\bx]\|_2^2 \leq \bbE\|\bx\|_2^2$, which holds for any random vector $\bx$.

Taking supremum in \eqref{minimax_worst-lb_p-geq-2_interim2} over all $(\eps,\infty)$-CLDP mechanisms $\{\calM_i:i\in[n]\}$ and estimators $\hatx$, we get
\begin{align}
r_{\eps,\infty,n}^{p,d} = \inf_{\lbrace \calM_i\in\calQ_{\left(\eps,\infty\right)}\rbrace}\inf_{\hatx}\sup_{\{\bx_i\}\in\calB_p^{d}}\bbE\left\|\frac{1}{n}\sum_{i=1}^{n}\bx_i-\hatx\left(\by^n\right)\right\|_{2}^{2} \geq \Omega\(d^{1-\frac{2}{p}}\min\left\lbrace 1,\frac{d}{n\min\{\eps,\eps^2\}}\right\rbrace\).
\end{align}

\subsubsection{Lower bound for $p\in[1,2]$}\label{subsubsec:minimax_worst-lb_p-leq-2} 
Similar to the argument given in Section~\ref{subsubsec:minimax_prob-lb_p-leq-2},
since $r_{\eps,\infty,n}^{p,d} \geq r_{\eps,\infty,n}^{1,d}$ holds for any $p\in[1,2]$, it suffices to lower-bound $r_{\eps,\infty,n}^{1,d}$.

Fix arbitrary $(\eps,\infty)$-CLDP mechanisms $\{\calM_i:i\in[n]\}$ and an estimator $\hatx$. 
It follows from \eqref{minimax_prob-lb_p-leq-2_interim3} that there exists a distribution $\bq\in\calP_p^d$, such that if we sample $\bx_i^{(q)}\sim\bq$, i.i.d.\ for all $i\in\left[n\right]$ and letting $\by_i=\calM_i(\bx_i^{(q)})$, we would have $\bbE\left\|\bmu_{\bq}-\hatx\left(\by^n\right)\right\|_{2}^{2}\geq \Omega\(\min\left\lbrace 1,\frac{d}{n\eps^2}\right\rbrace\)$. 
Now, by the same reasoning using which we obtained \eqref{minimax_worst-lb_p-geq-2_interim1}, we have
\begin{align}
\sup_{\{\bx_i\}\in\calB_p^{d}}\bbE\left\|\frac{1}{n}\sum_{i=1}^{n}\bx_i-\hatx\left(\by^n\right)\right\|_{2}^{2} 
&\geq \frac{1}{2}\bbE\left\|\bmu_{\bq}-\hatx\left(\by^n\right)\right\|_{2}^{2} -
\bbE\left\|\frac{1}{n}\sum_{i=1}^{n}\bx_i^{(q)}-\bmu_{\bq}\right\|_{2}^{2} \notag \\
&\stackrel{\text{(a)}}{\geq} \Omega\(\min\left\lbrace 1,\frac{d}{n\eps^2}\right\rbrace\) - \frac{1}{n} 
\stackrel{\text{(b)}}{\geq} \Omega\(\min\left\lbrace 1,\frac{d}{n\eps^2}\right\rbrace\) \label{minimax_worst-lb_p-leq-2_interim2}
\end{align}
In (a) we used 
\begin{equation}\label{1-norm-bnd}
\bbE\left\|\frac{1}{n}\sum_{i=1}^{n}\bx_i^{(q)}-\bmu_{\bq}\right\|_{2}^{2}\leq \frac{1}{n},
\end{equation} 
which can be obtained by first noting that for any $\bu\in\bbR^d$, we have $\|\bu\|_2\leq \|\bu\|_p$ for $p\in[1,2]$, and then using this in the set of inequalities which give \eqref{2-norm-bnd}. In (b), we assume $\eps\leq\calO(\sqrt{d})$.

Taking supremum in \eqref{minimax_worst-lb_p-geq-2_interim2} over all $(\eps,\infty)$-CLDP mechanisms $\{\calM_i:i\in[n]\}$ and estimators $\hatx$, we get $r_{\eps,\infty,n}^{1,d}\geq \Omega\(\min\left\lbrace 1,\frac{d}{n\eps^2}\right\rbrace\)$.

\subsection{Lower Bound on $r_{\epsilon,b,n}^{p,d}$: Proof of Theorem~\ref{thm:minimax_comm_lb}}\label{subsec:minimax_comm_lb_proof}
Let $M=2^{b}<d$ be the total number of possible outputs of the mechanism $\mathcal{R}$.
Let $\{o_1,o_2,\hdots,o_M\}$ be the set of $M$ possible outputs of $\mathcal{R}$.
For every $i\in[M]$, let $q_i=g(o_i)$.
We can write the $M$ possible outputs of $\mathcal{R}$ as columns of a $d\times M$ matrix $Q=\left[q_1,\ldots,q_M\right]$. Since $M<d$, the rank of the matrix $Q$ is at most $M$. Let $\bx\in\mathbb{R}^{d}$ be a vector in the null space of the matrix $Q$, i.e., $\bx^{T}q_j=0$ for all $j\in\left[M\right]$. Then, we set the sample of each client by $\bx_i=\overline{\bx}= \frac{\bx}{\|\bx\|_{p}}$ for all $i\in\left[n\right]$, and hence, $\bx_i\in\mathcal{B}_{p}^{d}$.  Observe that the estimator $\hat{\bx}=\frac{1}{n}\sum_{i=1}^{n}g\left(\mathcal{M}\left(\bx_i\right)\right)$ is in the column space of the matrix $Q$. Thus, we get
\begin{equation*}
r_{\epsilon,b,n}^{p,d}\geq\mathbb{E}\bigg\|\overline{\bx}-\frac{1}{n}\sum_{i=1}^{n}g\left(\mathcal{R}\left(\bx_i\right)\right)\bigg\|_{2}^{2} 
\stackrel{\left(a\right)}{=}\left\|\overline{\bx}\right\|_{2}^{2}+\mathbb{E}\bigg\|\frac{1}{n}\sum_{i=1}^{n}g\left(\mathcal{R}\left(\bx_i\right)\right)\bigg\|_{2}^{2} 
\geq \max\left\{1,d^{1-\frac{2}{p}}\right\}
\end{equation*} 
where step $\left(a\right)$ follows from the fact that $\overline{\bx}$ is in the null space of $Q$, while the estimator $\hat{\bx}$ is in the column space of $Q$.
This completes the proof of Theorem~\ref{thm:minimax_comm_lb}.

\subsection{Achievability for $\ell_1$-norm Ball: Proof of Theorem~\ref{thm:L1-norm_ub}}\label{subsec:L1-norm_ub_proof}
In this section, we propose an $\eps_0$-LDP mechanism that requires
$\calO\left(\log(d)\right)$-bits of communication per client using
private randomness and $1$-bit of communication per client using
public randomness. In other words we can guarantee
$(\eps_0,\calO\left(\log(d)\right))$-CLDP with private randomness and
$(\eps_0,1)$-CLDP using public randomness.
The proposed mechanism is based on the Hadamard matrix and is inspired from the Hadamard mechanism proposed by Acharya et al.~\cite{acharya2019hadamard}. We assume that $d$ is a power of $2$. Let $\bH_d$ denote the Hadamard matrix of order $d$, which can be constructed by the following recursive mechanism:
\[\bH_{d} = 
\begin{bmatrix}
\bH_{\nicefrac{d}{2}} & \hfill \bH_{\nicefrac{d}{2}} \\
\bH_{\nicefrac{d}{2}} & -\bH_{\nicefrac{d}{2}}
\end{bmatrix} 
\qquad\qquad
\bH_1 = 
\begin{bmatrix}
1
\end{bmatrix}
\]
Client $i$ has an input $\bx_i\in\calB_1^d\left(a\right)$. It computes $\by_i=\frac{1}{\sqrt{d}}\bH_d\bx_i$.
Note that each coordinate of $\by_i$ lies in the interval $\left[-\nicefrac{a}{\sqrt{d}}, \nicefrac{a}{\sqrt{d}}\right]$.
Client $i$ selects $j\sim \textsf{Unif}\left[d\right]$ and quantize $y_{i,j}$ privately according to \eqref{L1-hadamard_1-bit-quantizer} and obtains $\bz_i\in\big\{\pm a\bH_d(j)\big(\frac{e^{\eps_0}+1}{e^{\eps_0}-1}\big)\big\}$, which can be represented using only $1$-bit. Here, $\bH_d(j)$ denotes the $j$-th column of the Hadamard matrix $\bH_d$. Server receives the $n$ messages $\lbrace\bz_1,\ldots,\bz_n\rbrace$ from the clients and outputs their average $\frac{1}{n}\sum_{i=1}^{n}\bz_i$. 
We present this mechanism in Algorithm~\ref{R1-quantize-client} -- we only present the client-side part of the algorithm, as server only averages the messages received from the clients.

\begin{algorithm}
\caption{\textsf{$\ell_1$-MEAN-EST} ($\calR_1$: the client-side algorithm)}
\label{R1-quantize-client}
\begin{algorithmic}[1]
\State \textbf{Input:} Vector $\bx\in\calB_{1}^{d}\left(a\right)$, and local privacy level $\eps_0>0$. 
\State Construct $\by=\frac{1}{\sqrt{d}}\bH_d\bx$
\State Sample $j\sim \textsf{Unif}[d]$ and quantize $y_j$ as follows:
\begin{align}\label{L1-hadamard_1-bit-quantizer}
\bz=\left\{\begin{array}{rl}
+a\bH_d\left(j\right)\left(\frac{e^{\eps_0}+1}{e^{\eps_0}-1}\right)& \text{ w.p. } 
\frac{1}{2} + \frac{\sqrt{d}y_j}{2a}\frac{e^{\eps_0}-1}{e^{\eps_0}+1} \\
-a\bH_d\left(j\right)\left(\frac{e^{\eps_0}+1}{e^{\eps_0}-1}\right)& \text{ w.p. }
\frac{1}{2} - \frac{\sqrt{d}y_j}{2a}\frac{e^{\eps_0}-1}{e^{\eps_0}+1}
\end{array} \right.
\end{align} 
\State Return $\bz$.
\end{algorithmic}
\end{algorithm}
\begin{lemma}\label{R1-quantize-properties} 
The mechanism $\calR_1$ presented in Algorithm~\ref{R1-quantize-client} satisfies the following properties, where $\eps_0>0$:
\begin{enumerate}
\item $\calR_1$ is $\left(\eps_0,\log\left(d\right)+1\right)$-CLDP and requires only $1$-bit of communication using public randomness. 
\item $\calR_1$ is unbiased and has bounded variance, i.e., for every $\bx\in\calB_1^d\left(a\right)$, we have
$$\bbE\left[\calR_1\left(\bx\right)\right]=\bx \quad \text{and} \quad \bbE\|\calR_1\left(\bx\right)-\bx\|_{2}^{2} \leq a^2d\left(\frac{e^{\eps_0}+1}{e^{\eps_0}-1}\right)^{2}.$$
\end{enumerate}
\end{lemma}
We prove Lemma~\ref{R1-quantize-properties} in Appendix~\ref{app:L1-norm-algo1}.

Now we are ready to prove Theorem~\ref{thm:L1-norm_ub}.
Let $\calR_1(\bx)$ denote the output of Algorithm~\ref{R1-quantize-client} on input $\bx$.
As mentioned above, the server employs a simple estimator that simply averages the $n$ received messages, i.e., the server outputs $\hatx(\bz^n)=\frac{1}{n}\sum_{i=1}^n\bz_i=\frac{1}{n}\sum_{i=1}^n\calR_1(\bx_i)$. In the following, first we show the bound on $r_{\eps_0,b,n}^{1,d}\left(a\right)$ and then on $R_{\eps_0,b,n}^{1,d}\left(a\right)$ for $b=\log(d)+1$.
\begin{align}
\text{For }r_{\eps_0,b,n}^{1,d}\left(a\right): \qquad & \sup_{\{\bx_i\}\in\calB_1^{d}\left(a\right)}\bbE\left\|\barx - \hatx(\bz^{n})\right\|_{2}^{2} = 
\sup_{\{\bx_i\}\in\calB_1^{d}\left(a\right)}\bbE\left\|\frac{1}{n}\sum_{i=1}^n\(\bx_i-\calR_1(\bx_i)\)\right\|_{2}^{2} \notag \\
&\qquad\stackrel{\text{(a)}}{=} \sup_{\{\bx_i\}\in\calB_1^{d}\left(a\right)} \frac{1}{n^2}\sum_{i=1}^n\bbE\left\|\bx_i-\calR_1(\bx_i)\right\|_{2}^{2} 
\ \stackrel{\text{(b)}}{\leq}\ \frac{a^2d}{n}\left(\frac{e^{\eps_0}+1}{e^{\eps_0}-1}\right)^{2}, \label{L1-norm_ub_proof_interim2}
\end{align}
where (a) uses the fact that all clients use independent private randomness (which makes the random variables $\bx_i-\calR_1(\bx_i)$ independent for different $i$'s and also that $\calR_1$ is unbiased. (b) uses that $\calR_1$ has bounded variance.
Taking infimum in \eqref{L1-norm_ub_proof_interim2} over all $(\eps_0,b)$-CLDP mechanisms (where $b=\log(d)+1$) and estimators $\hatx$, we have that $r_{\eps_0,b,n}^{1,d}\left(a\right) \leq \frac{a^2d}{n}\left(\frac{e^{\eps_0}+1}{e^{\eps_0}-1}\right)^{2}$,
which is $\calO\(\frac{a^2d}{n\eps_0^2}\)$ when $\eps_0=\calO(1)$.

\begin{align}
\text{For }R_{\eps_0,b,n}^{1,d}\left(a\right): \qquad \sup_{\bq\in\calP_1^d\left(a\right)}\bbE\left\|\bmu_{\bq} - \hatx(\bz^n)\right\|_2^2
&\stackrel{\text{(c)}}{\leq} \sup_{\bq\in\calP_1^d\left(a\right)}\left[2\bbE\left\|\bmu_{\bq} - \barx\right\|_2^2 + 2\bbE\left\|\barx - \hatx(\bz^n)\right\|_2^2 \right] \notag \\
&\stackrel{\text{(d)}}{\leq} \frac{2a^2}{n} + \frac{2a^2d}{n}\left(\frac{e^{\eps_0}+1}{e^{\eps_0}-1}\right)^{2} \label{L1-norm_ub_proof_interim3}
\end{align}
In the LHS of (c), for any $\bq\in\calP_1^d\left(a\right)$, first we generate $n$ i.i.d.\ samples $\bx_1,\hdots,\bx_n$ and then compute $\bz_i=\calR_1(\bx_i)$ for all $i\in[n]$. We use the Jensen's inequality in (c). We used $\bbE\left\|\bmu_{\bq} - \barx\right\|_2^2 \leq \frac{a^2}{n}$ (see \eqref{1-norm-bnd}) in (d).
Taking infimum in \eqref{L1-norm_ub_proof_interim3} over all $(\eps_0,b)$-CLDP mechanisms (where $b=\log(d)+1$) and estimators $\hatx$, we have that $R_{\eps_0,b,n}^{1,d}\left(a\right) \leq \frac{2a^2}{n} + \frac{2a^2d}{n}\left(\frac{e^{\eps_0}+1}{e^{\eps_0}-1}\right)^{2}$, which is $\calO\(\frac{a^2d}{n\eps_0^2}\)$ when $\eps_0=\calO(1)$.

This completes the proof of Theorem~\ref{thm:L1-norm_ub}.

\subsection{Achievability for $\ell_2$-norm Ball: Proof of Theorem~\ref{thm:L2-norm_ub}}\label{subsec:L2-norm_ub_proof}
In this section, we propose an $\eps_0$-LDP mechanism that requires $\mathcal{O}\left(d\right)$-bits of communication per client using private randomness. Our proposed mechanism is a combination of the private-mechanism \textsf{Priv} of Duchi et al.~\cite[Section $4.2.3$]{duchi2018minimax} and the non-private quantization mechanism \textsf{Quan} of Mayekar and Tyagi~\cite[Section $4.2$]{mayekar2020limits}. 
For completeness, we describe both these mechanisms in Algorithm~\ref{l-2-privacy} and Algorithm~\ref{l-2-quantize}, respectively, and our proposed mechanism in Algorithm~\ref{R_2-quantize-client}.
Each client $i$ first privatize its input $\bx_i\in\calB_2^d\left(a\right)$ using \textsf{Priv} and then quantize the privatized result using \textsf{Quan} and sends the final result  $\bz_i=\textsf{Quan}(\textsf{Priv}(\bx_i))$ to the server, which outputs the average of all the received $n$ messages. Since the server is only taking an average of the received messages, we only present the client side of our mechanism in Algorithm~\ref{R_2-quantize-client}.

\begin{algorithm}
\caption{\textsf{$\ell_2$-MEAN-EST} ($\calR_2$: the client-side algorithm)}
\label{R_2-quantize-client}
\begin{algorithmic}[1]
\State \textbf{Input:} Vector $\bx\in\calB_{2}^{d}\left(a\right)$, and local privacy level $\eps_0>0$. 
\State Apply the randomized mechanism $\by=\textsf{Priv}\left(\bx\right)$.
\State Return $\bz=\textsf{Quan}\left(\by\right)$. 
\end{algorithmic}
\end{algorithm}

\begin{algorithm}
\caption{\textsf{Priv} (a private mechanism from \cite{duchi2018minimax})}
\label{l-2-privacy}
\begin{algorithmic}[1]
\State \textbf{Input:} Vector $\bx\in\calB_{2}^{d}\left(a\right)$, and local privacy level $\eps_0>0$. 
\State Compute $\widetilde{\bx}=\left\{\begin{array}{ll}
\hfill +a\frac{\bx}{\|\bx\|_2}&\ \text{w.p.}\ \frac{1}{2}+\frac{\|\bx\|_{2}}{2a}\\
-a\frac{\bx}{\|\bx\|_2}&\ \text{w.p.}\ \frac{1}{2}-\frac{\|\bx\|_{2}}{2a}
\end{array}\right.$
\State Sample $U\sim\text{Bernoulli}\left(\frac{e^{\eps_0}}{e^{\eps_0}+1}\right)$
\State $M \triangleq a\frac{\sqrt{\pi}}{2}\frac{\Gamma\left(\frac{d-1}{2}+1\right)}{\Gamma\left(\frac{d}{2}+1\right)}\frac{e^{\eps_0}+1}{e^{\eps_0}-1}$
\State $\bz=
\begin{cases}
\textsf{Unif}\left(\by:\by^{T}\widetilde{\bx}>0,\|\by\|_2=M\right)& \text{if}\ U=1\\
\textsf{Unif}\left(\by:\by^{T}\widetilde{\bx}\leq0,\|\by\|_2=M\right)& \text{if}\ U=0
\end{cases}$ 
\State Return $\bz$.
\end{algorithmic}
\end{algorithm}
\begin{lemma}[{\cite[Appendix~$I.2$]{duchi2018minimax}}]\label{l-2-privacy-properties} 
The mechanism \textsf{\em Priv} presented in Algorithm~\ref{l-2-privacy} is unbiased and outputs a bounded length vector, i.e., for every $\bx\in\calB_2^d\left(a\right)$, we have
\begin{equation*}
\bbE[\textsf{\em Priv}(\bx)] = \bx \quad \text{and} \quad \|\textsf{\em Priv}(\bx)\|_{2}^{2} = M^2 \leq a^2 d\(\frac{3\sqrt{\pi}}{4}\frac{e^{\eps_0}+1}{e^{\eps_0}-1}\)^2.
\end{equation*} 
\end{lemma}

\begin{algorithm}
\caption{\textsf{Quan} (a quantization mechanism from \cite{mayekar2020limits})}
\label{l-2-quantize}
\begin{algorithmic}[1]
\State \textbf{Input:} Vector $\bx\in\calB_{2}^{d}\left(a\right)$, where $a$ is the radius of the ball.
\State Compute $\widetilde{\bx}=\left\{\begin{array}{ll}
\hfill \frac{\bx}{\|\bx\|_1}&\ \text{w.p.}\ \frac{1+\|\bx\|_1}{2 a\sqrt{d}}\\
-\frac{\bx}{\|\bx\|_1}&\ \text{w.p.}\ \frac{1-\|\bx\|_1}{2 a\sqrt{d}}
\end{array}\right.$
\State Generate a discrete distribution $\bmu=\left(|\widetilde{x}_1|,\ldots,|\widetilde{x}_d|\right)$ where $\Pr[\bmu=i]=|\widetilde{x}_i|$.
\State Construct a $d$-dimensional vector $\by$ by sampling $y_j\sim \bmu$ for $j\in\left[d\right]$ 
\State Return $\bz=\frac{1}{d}\sum_{j=1}^{d}\(a\sqrt{d}\cdot\mathrm{sgn}(\widetilde{x}_{y_j})\cdot\be_{y_j}\)$.
\end{algorithmic}
\end{algorithm}
\begin{lemma}[{\cite[Theorem~$4.2$]{mayekar2020limits}}]\label{l-2-quan-properties} 
The mechanism $\textsf{\em Quan}$ presented in Algorithm~\ref{l-2-quantize} is unbiased and has bounded variance, i.e., for every $\bx\in\calB_2^d(a)$, we have
\begin{equation*}
\bbE[\textsf{\em Quan}(\bx)] = \bx \quad \text{and} \quad \bbE\|\textsf{\em Quan}(\bx)-\bx\|_{2}^{2} \leq 2\|\bx\|^2 \leq 2a^2.
\end{equation*} 
Furthermore, it requires $d\(\log(e)+1\)$-bits to represent its output.
\end{lemma}
Note that the radius $a$ in Lemma~\ref{l-2-quan-properties} is equal to the length of any output of \textsf{Priv}, which is $M$ (see line 4 of Algorithm~\ref{l-2-privacy}).

\begin{lemma}\label{R_2-quantize-properties}
The mechanism $\calR_2$ presented in Algorithm~\ref{R_2-quantize-client} satisfies the following properties, where $\eps_0>0$:
\begin{enumerate}
\item $\calR_2$ is $\left(\eps_0,d(\log(e)+1)\right)$-CLDP.
\item $\calR_2$ is unbiased and has bounded variance, i.e., for every $\bx\in\calB_2^d\left(a\right)$, we have 
$$\bbE\left[\calR_2\left(\bx\right)\right]=\bx \quad \text{and} \quad \bbE\|\calR_2\left(\bx\right)-\bx\|_{2}^{2}\leq 6 a^2 d\(\frac{e^{\eps_0}+1}{e^{\eps_0}-1}\)^{2}.$$
\end{enumerate}
\end{lemma}
We prove Lemma~\ref{R_2-quantize-properties} in Appendix~\ref{app:L-infty-norm_ub_proof}.

Now we are ready to prove Theorem~\ref{thm:L2-norm_ub}.
In order to bound $r_{\eps_0,b,n}^{2,d}\left(a\right)$ for $b= d(\log(e)+1)$, we follow exactly the same steps that we used to bound $r_{\eps_0,b,n}^{1,d}\left(a\right)$ and arrived at \eqref{L1-norm_ub_proof_interim2}. This would give $r_{\eps_0,b,n}^{2,d}\left(a\right)\leq  \frac{6a^2 d}{n}\left(\frac{e^{\eps_0}+1}{e^{\eps_0}-1}\right)^{2}$, which is $\calO\(\frac{a^2 d}{n\eps_0^2}\)$ when $\eps_0=\calO(1)$.
To bound $R_{\eps_0,b,n}^{2,d}\left(a\right)$, first note that when $\bx_1,\hdots,\bx_n\in\calB_{2}^d\left(a\right)$, then we have from \eqref{1-norm-bnd} that $\bbE\left\|\bmu_{\bq} - \barx\right\|_2^2 \leq \frac{a^2}{n}$. Here $\bq\in\calP_{2}^d\left(a\right)$ and $\bx_1,\hdots,\bx_n$ are sampled from $\bq$ i.i.d. Now, following exactly the same steps that we used to bound $R_{\eps_0,b,n}^{1,d}\left(a\right)$ and arrived at \eqref{L1-norm_ub_proof_interim3}. This would give $R_{\eps_0,b,n}^{2,d}\left(a\right) \leq \frac{2a^2}{n} + \frac{12 a^2 d}{n}\left(\frac{e^{\eps_0}+1}{e^{\eps_0}-1}\right)^{2}$ for $b= d(\log(e)+1)$. Note that $R_{\eps_0,b,n}^{2,d}\left(a\right)=\calO\(\frac{a^2 d}{n\eps_0^2}\)$ when $\eps_0=\calO(1)$.

This completes the proof of Theorem~\ref{thm:L2-norm_ub}.

\subsection{Achievability for $\ell_{\infty}$-norm Ball: Proof of Theorem~\ref{thm:L-infty-norm_ub}}\label{subsec:L-infty-norm_ub_proof}
In this section, we propose an $\eps_0$-LDP mechanism that requires $\mathcal{O}\left(\log\left(d\right)\right)$-bits per client using private randomness and $1$-bit of communication per client using public randomness. 
Each client $i$ has an input $\bx_i\in\calB_{\infty}^d\left(a\right)$.
It selects $j\sim \textsf{Unif}[d]$ and quantize $x_{i,j}$ according to \eqref{L-infty-1-bit-quantizer} and obtains $\bz_{i}\in\big\lbrace \pm a d\big(\frac{e^{\eps_0}+1}{e^{\eps_0}-1}\big)\be_j\big\rbrace$, which can be represented using only $1$ bit, where $\be_j$ is the $j$'th standard basis vector in $\bbR^d$. Client $i$ sends $\bz_i$ to the server. 
Server receives the $n$ messages $\lbrace\bz_1,\ldots,\bz_n\rbrace$ from the clients and outputs their average $\frac{1}{n}\sum_{i=1}^{n}\bz_i$. 
We present this mechanism in Algorithm~\ref{R_infty-quantize-client} -- we only present the client-side part of the algorithm, as server only averages the messages received from the clients.

\begin{algorithm}
\caption{\textsf{$\ell_\infty$-MEAN-EST} ($\calR_{\infty}$: the client-side algorithm)}
\label{R_infty-quantize-client}
\begin{algorithmic}[1]
\State \textbf{Input:} Vector $\bx\in\calB_{\infty}^{d}\left(a\right)$, and local privacy level $\eps_0>0$. 
\State Sample $j\sim \textsf{Unif}[d]$ and quantize $x_j$ as follows:
\begin{align}\label{L-infty-1-bit-quantizer}
\bz=\left\{\begin{array}{rl}
+a d\left(\frac{e^{\eps_0}+1}{e^{\eps_0}-1}\right)\be_j& \text{ w.p. } \frac{1}{2}+\frac{x_j}{2a}\frac{e^{\eps_0}-1}{e^{\eps_0}+1} \\
-a d\left(\frac{e^{\eps_0}+1}{e^{\eps_0}-1}\right)\be_j& \text{ w.p. } \frac{1}{2}-\frac{x_j}{2a}\frac{e^{\eps_0}-1}{e^{\eps_0}+1} 
\end{array} \right.
\end{align}
where $\be_j$ is the $j$'th standard basis vector in $\bbR^d$
\State Return $\bz$.
\end{algorithmic}
\end{algorithm}
\begin{lemma}\label{R_infty-quantize-properties} 
The mechanism $\calR_\infty$ presented in Algorithm~\ref{R_infty-quantize-client} satisfies the following properties, where $\eps_0>0$:
\begin{enumerate}
\item $\calR_\infty$ is $\left(\eps_0,\log\left(d\right)+1\right)$-CLDP and requires only $1$-bit of communication using public randomness. 
\item $\calR_{\infty}$ is unbiased and has bounded variance, i.e., for every $\bx\in\calB_{\infty}^d\left(a\right)$, we have
$$\bbE\left[\calR_\infty\left(\bx\right)\right]=\bx \quad \text{and} \quad \bbE\|\calR_\infty\left(\bx\right)-\bx\|_{2}^{2}\leq a^2 d^{2}\left(\frac{e^{\eps_0}+1}{e^{\eps_0}-1}\right)^{2}.$$
\end{enumerate}
\end{lemma}
We prove Lemma~\ref{R_infty-quantize-properties} in Appendix~\ref{app:L-infty-norm-algo}.

Now we are ready to prove Theorem~\ref{thm:L-infty-norm_ub}.
In order to bound $r_{\eps_0,b,n}^{\infty,d}\left(a\right)$ for $b= \log\left(d\right)+1$, we follow exactly the same steps that we used to bound $r_{\eps_0,b,n}^{1,d}\left(a\right)$ and arrived at \eqref{L1-norm_ub_proof_interim2}. This would give $r_{\eps_0,b,n}^{\infty,d}\left(a\right)\leq \frac{a^2 d^2}{n}\left(\frac{e^{\eps_0}+1}{e^{\eps_0}-1}\right)^{2}$, which is $\calO\(\frac{a^2 d^2}{n\eps_0^2}\)$ when $\eps_0=\calO(1)$.
To bound $R_{\eps_0,b,n}^{\infty,d}\left(a\right)$, first note that when $\bx_1,\hdots,\bx_n\in\calB_{\infty}^d\left(a\right)$, then we have from~\eqref{2-norm-bnd} (by substituting $p=\infty$) that $\bbE\left\|\bmu_{\bq} - \barx\right\|_2^2 \leq \frac{a^2 d}{n}$. Here $\bq\in\calP_{\infty}^d\left(a\right)$ and $\bx_1,\hdots,\bx_n$ are sampled from $\bq$ i.i.d. Now, following exactly the same steps that we used to bound $R_{\eps_0,b,n}^{1,d}\left(a\right)$ and arrived at \eqref{L1-norm_ub_proof_interim3}. This would give $R_{\eps_0,b,n}^{\infty,d}\left(a\right)\leq \frac{2a^2 d}{n} + \frac{2a^2d^2}{n}\left(\frac{e^{\eps_0}+1}{e^{\eps_0}-1}\right)^{2,d}$ for $b= \log\left(d\right)+1$.
Note that $R_{\eps_0,b,n}^{\infty,d}\left(a\right)=\calO\(\frac{a^2 d^2}{n\eps_0^2}\)$ when $\eps_0=\calO(1)$.

This completes the proof of Theorem~\ref{thm:L-infty-norm_ub}.

\subsection{Achievability for $\ell_{p}$-norm Ball for $p\in[1,\infty)$: Proof of Corollary~\ref{corol:L-p-norm_ub}}\label{subsec:L-p-norm_ub_proof}
In this section, first we propose two $\epsilon_0$-LDP mechanisms for $\ell_{p}$-norm ball $\calB_p^d(a)$ for $p\in[1,\infty)$ based on the inequalities between different norms, and our final mechanism will be chosen probabilistically from these two.
The first mechanism, which we denote by $\calR_p^{(1)}$, is based on the private mechanism $\calR_1$ (presented in Algorithm~\ref{R1-quantize-client}) that requires $\mathcal{O}\left(\log\left(d\right)\right)$ bits per client. The second mechanism, which we denote by $\calR_p^{(2)}$ is based on the private mechanism $\calR_2$ (presented in Algorithm~\ref{R_2-quantize-client}) that requires $\mathcal{O}\left(d\right)$ bits per client. 
Observe that for any $1\leq p\leq q\leq\infty $, using the relation between different norms ($\|\bu\|_q \leq \|\bu\|_p \leq d^{\frac{1}{p}-\frac{1}{q}}\|\bu\|_q$), we have
\begin{equation}
\calB_{q}^{d}\left(a\right)\subseteq \calB_{p}^{d}\left(a\right)\subseteq \calB_{q}^{d}\left(ad^{\frac{1}{p}-\frac{1}{q}} \right).
\end{equation} 

\begin{enumerate}
\item {\it Description of the private mechanism $\calR_p^{(1)}$:} Each client has a vector $\bx_i\in\calB_{p}^{d}\left(a\right)\subseteq \calB_{1}^{d}\left(ad^{1-\frac{1}{p}} \right)$. Thus, each client runs the private mechanism $\calR_1\left(\bx_i\right)$  presented in Algorithm~\ref{R1-quantize-client} with radius $ad^{1-\frac{1}{p}} $. Thus, the mechanism $\calR_p^{(1)}$ for $p\in[1,\infty)$ satisfies the following properties, where $\eps_0>0$:
\begin{itemize}
\item $\calR_p^{(1)}$ is $\left(\eps_0,\log\left(d\right)+1\right)$-CLDP and requires only $1$-bit of communication using public randomness. 
\item $\calR_p^{(1)}$ is unbiased and has bounded variance, i.e., for every $\bx\in\calB_p^d\left(a\right)$, we have
$$\bbE\left[\calR_p^{(1)}\left(\bx\right)\right]=\bx \quad \text{and} \quad \bbE\|\calR_p^{(1)}\left(\bx\right)-\bx\|_{2}^{2} \leq a^2d^{3-\frac{2}{p}}\left(\frac{e^{\eps_0}+1}{e^{\eps_0}-1}\right)^{2}.$$
\end{itemize}

\item {\it Description of the private mechanism $\calR_p^{(2)}$:} Each client has a vector $\bx_i\in\calB_{p}^{d}\left(a\right)\subseteq \calB_{2}^{d}\left(a\max\lbrace d^{\frac{1}{2}-\frac{1}{p}},1\rbrace \right)$. 
Thus, each client runs the private mechanism $\calR_2\left(\bx_i\right)$ presented in Algorithm~\ref{R_2-quantize-client} with radius $a\max\lbrace d^{\frac{1}{2}-\frac{1}{p}}, 1\rbrace$. Thus, the mechanism $\calR_p^{(2)}$ for $p\in[1,\infty)$ satisfies the following properties, where $\eps_0>0$:
\begin{itemize}
\item $\calR_p^{(2)}$ is $\left(\eps_0,d\left(\log\left(e\right)+1\right)\right)$-CLDP. 
\item $\calR_p^{(2)}$ is unbiased and has bounded variance, i.e., for every $\bx\in\calB_p^d\left(a\right)$, we have
$$\bbE\left[\calR_p^{(2)}\left(\bx\right)\right]=\bx \quad \text{and} \quad \bbE\|\calR_p^{(2)}\left(\bx\right)-\bx\|_{2}^{2} \leq 6a^2\max\lbrace d^{2-\frac{2}{p}},d\rbrace\left(\frac{e^{\eps_0}+1}{e^{\eps_0}-1}\right)^{2}.$$
\end{itemize}
\end{enumerate}
Note that $\calR_p^{(1)}$ requires low communication and has high variance, whereas, $\calR_p^{(2)}$ requires high communication and has low variance: $\calR_p^{(2)}$ requires exponentially more communication than $\calR_p^{(1)}$, whereas, $\calR_p^{(1)}$ has a factor of $d$ more variance than $\calR_p^{(2)}$.

To define our final mechanism $\calR_p$ for any norm $p\in[1,\infty)$, we choose $\calR_p^{(1)}$  with probability $\bar{p}$ and $\calR_p^{(2)}$ with probability $(1-\bar{p})$, where $\bar{p}$ is any number in $[0,1]$. 
Note that $\calR_p$ is $\eps_0$-LDP and requires $\bar{p}\log(d)+(1-\bar{p})d\log(e)+1$ expected communication, where expectation is taken over the sampling of choosing $\calR_p^{(1)}$ or $\calR_p^{(2)}$. We have the following bounds on $r_{\eps_0,b,n}^{p,d}(a)$ and $R_{\eps_0,b,n}^{p,d}(a)$:
\begin{align*}
r_{\eps_0,b,n}^{p,d}(a) \ &\leq \ \bar{p}\,d^{2-\frac{2}{p}}r_{\eps_0,b,n}^{1,d}(a) + (1-\bar{p})\max\lbrace d^{1-\frac{2}{p}}, 1\rbrace r_{\eps_0,b,n}^{2,d}(a) \\
\text{For }R_{\eps_0,b,n}^{p,d}(a) \ &\leq \ \bar{p}\,d^{2-\frac{2}{p}}R_{\eps_0,b,n}^{1,d}(a) + (1-\bar{p})\max\lbrace d^{1-\frac{2}{p}}, 1\rbrace R_{\eps_0,b,n}^{2,d}(a)
\end{align*}

This completes the proof of Corollary~\ref{corol:L-p-norm_ub}.
\section{Optimization: Privacy, Communication, and Convergence Analyses}\label{sec:OptPerf}

In this section, we establish the privacy, communication, and convergence guarantees of Algorithm~\ref{algo:optimization-algo} and prove Theorem~\ref{thm:main-opt-result}. We show these three results on privacy, communication, and convergence separately in the next three subsections.

\subsection{Proof of Theorem~\ref{thm:main-opt-result}: Privacy}\label{subsec:main-result_priv}
Recall from Algorithm~\ref{algo:optimization-algo} that each client applies the compressed LDP mechanism $\calR_p$ (hereafter denoted by $\calR$, for simplicity) with privacy parameter $\eps_0$ on each gradient. This implies that the mechanism $\mathcal{A}_{cldp}$ guarantees local differential privacy $\epsilon_0$ for each sample $d_{ij}$ per epoch. Thus, it remains to analyze the central DP of the mechanism $\mathcal{A}_{cldp}$.

Fix an iteration number $t\in[T]$. Let $\calM_t\left(\theta_t,\calD\right)$ denote the private mechanism at time $t$ that takes the dataset $\calD$ and an auxiliary input $\theta_t$ (which is the parameter vector at the $t$'th iteration) and generates the parameter $\theta_{t+1}$ as an output. 
Recall that the input dataset at client $i\in[m]$ is denoted by $\calD_i=\{d_{i1},d_{i2},\hdots,d_{ir}\}\in\frS^r$ and 
$\calD=\bigcup_{i=1}^m\calD_i$ denotes the entire dataset.
Thus, the mechanism $\calM_t$ on any input dataset $\calD=\bigcup_{i=1}^{m}\mathcal{D}_i\in\frS^n$ can be defined as:
\begin{align}
\calM_t(\theta_t;\calD) = \calH_{ks} \circ \samp_{m,k} \(\calG_1,\ldots,\calG_m\), \label{mech-t}
\end{align}
where $\calG_i=\samp_{r,s} \left(\calR(\bx_{i1}^t),\hdots,\calR(\bx_{ir}^t)\right)$ and $\bx_{ij}^t=\nabla_{\theta_t} f(\theta_t;d_{ij}), \forall i\in[m],j\in[r]$. Here, $\calH_{ks}$ denotes the shuffling operation on $ks$ elements and $\samp_{m,k}$ denotes the sampling operation for choosing a random subset of $k$ elements from a set of $m$ elements.

For convenience, in the rest of the proof, we suppress the auxiliary input $\theta_t$ and simply denote $\calM_t(\theta_t;\calD)$ by $\calM_t(\calD)$. We can do this because $\theta_t$ only affects the gradients, and the analysis in this section is for an arbitrary set of gradients.

In the following lemma, we state the privacy guarantee of the mechanism $\calM_t$ for each $t\in[T]$.
\begin{lemma}\label{lemm:priv_opt} 
Let $s=1$ and $q=\frac{k}{mr}$. Suppose $\calR$ is an $\epsilon_0$-LDP mechanism, where $\eps_0\leq \frac{\log\(qn/\log\left(1/\tilde{\delta}\right)\)}{2}$ and $\tilde{\delta}>0$ is arbitrary. Then, for any $t\in\left[T\right]$, the mechanism $\calM_t$ is $\left(\epsbar,\delbar\right)$-DP, where $\epsbar=\ln(1+q(e^{\tilde{\eps}}-1)),\delbar=q\tilde{\delta}$ with
$\tilde{\epsilon} = \mathcal{O}\left(\min\{\eps_0,1\}e^{\eps_0}\sqrt{\frac{\log\left(1/\tilde{\delta}\right)}{qn}}\right)$. In particular, if $\epsilon_0=\mathcal{O}\left(1\right)$, we get $\overline{\epsilon}=\mathcal{O}\left(\epsilon_0\sqrt{\frac{q\log\left(1/\tilde{\delta}\right)}{n}}\right)$.
\end{lemma}
We prove Lemma~\ref{lemm:priv_opt} in
Appendix~\ref{app:supp_composition_priv}.  In the statement of
Lemma~\ref{lemm:priv_opt}, we are amplifying the privacy by using the
subsampling as well as shuffling ideas.  For subsampling, note that we
do not pick a uniformly random subset of size $ks$ from $n$
points. So, we cannot directly apply the amplification by subsampling
result stated in Lemma~\ref{lemma:amplification_sampling}. However, as
it turns out that the only property we will need for privacy
amplification by subsampling is that each data point is picked by
probability $q=\frac{ks}{mr}$, which holds true in our setting. See
Appendix~\ref{app:supp_composition_priv} for more details.

Note that the Algorithm $\calA_{cldp}$ is a sequence of $T$ adaptive mechanisms $\calM_1,\hdots,\calM_T$, where each $\calM_t$ for $t\in[T]$ satisfies the privacy guarantee stated in Lemma~\ref{lemm:priv_opt}.
 Now, we invoke the strong composition stated in Lemma~\ref{lemm:strong-composition} to obtain the privacy guarantee of the algorithm $\calA_{cldp}$. We can conclude that for any $\delta'>0$, $\calA_{cldp}$ is $\left(\epsilon,\delta\right)$-DP for
\begin{align*}
\epsilon=\sqrt{2T\log\left(1/\delta^{\prime}\right)}\overline{\epsilon}+T\overline{\epsilon}\left(e^{\overline{\epsilon}}-1\right),\quad \delta=qT\tilde{\delta}+\delta^{\prime},
\end{align*}
where $\epsbar$ is from Lemma~\ref{lemm:priv_opt}.
We have from Lemma~\ref{lemm:strong-composition} that if $\epsbar=\mathcal{O}\left(\sqrt{\frac{\log\left(1/\delta^{\prime}\right)}{T}}\right)$, then $\epsilon=\mathcal{O}\left(\overline{\epsilon}\sqrt{T\log\left(1/\delta^{\prime}\right)}\right)$.
If $\eps_0=\calO(1)$, then we can satisfy this condition on $\epsbar$ by choosing $\eps_0=\calO\(\sqrt{\frac{n\log(1/\delta')}{qT\log(1/\tilde{\delta})}}\)$. By substituting the bound on $\overline{\epsilon}=\mathcal{O}\left(\epsilon_0\sqrt{\frac{q\log\left(1/\tilde{\delta}\right)}{n}}\right)$ from Lemma~\ref{lemm:priv_opt}, we have $\eps=\mathcal{O}\left(\epsilon_0\sqrt{\frac{qT\log\left(1/\tilde{\delta}\right)\log\left(1/\delta'\right)}{n}}\right)$.
By setting $\tilde{\delta}=\frac{\delta}{2qT}$ and $\delta'=\frac{\delta}{2}$, we get $\eps_0=\calO\(\sqrt{\frac{n\log(2/\delta)}{qT\log(2qT/\delta)}}\)$ and $\epsilon = \mathcal{O}\left(\epsilon_0\sqrt{\frac{qT\log\left(2qT/\delta\right)\log\left(2/\delta\right)}{n}}\right)$.
This completes the proof of the privacy part of Theorem~\ref{thm:main-opt-result}.

\subsection{Proof of Theorem~\ref{thm:main-opt-result}: Communication}\label{subsec:optimization_communication}
The $(\eps_0,b)$-CLDP mechanism $\calR_p:\calX\to\calY$ used in Algorithm~\ref{algo:optimization-algo} has output alphabet $\calY=\{1,2,\hdots,B=2^b\}$. So, the output of $\calR_p$ on any input can be represented by $b$ bits. Therefore, the na\"ive scheme for any client to send the $s$ compressed and private gradients requires $sb$ bits per iteration.
We can reduce this communication cost by using the histogram trick from \cite{mayekar2020limits} which was applied in the context of non-private quantization. The idea is as follows.
Since any client applies the {\em same} randomized mechanism $\mathcal{R}_p$ to the $s$ gradients, the output of these $s$ identical mechanisms can be represented accurately using the histogram of the $s$ outputs, which takes value from the set $\mathcal{A}_{B}^s = \lbrace \left(n_1,\ldots,n_{B}\right): \sum_{j=1}^{B}n_j=s \text{ and } n_j\geq0,\forall j\in[B] \rbrace$. Since the cardinality of this set is ${s+B-1 \choose s}\leq \left(\frac{e\left(s+B-1\right)}{s}\right)^{s}$, it requires at most $s\left(\log\left(e\right)+\log\left(\frac{s+B-1}{s}\right)\right)$ bits to send the $s$ compressed gradients. Since the probability that the client is chosen at any time $t\in\left[T\right]$ is given by $\frac{k}{m}$, the expected number of bits per client in Algorithm $\mathcal{A}_{cldp}$ is given by $\frac{k}{m}\times T\times s\left(\log\left(e\right)+\log\left(\frac{s+B-1}{s}\right)\right)$ bits, where expectation is taken over the sampling of $k$ out of $m$ clients in all $T$ iterations.

This completes the proof of the second part of Theorem~\ref{thm:main-opt-result}.
  
\subsection{Proof of Theorem~\ref{thm:main-opt-result} : Convergence} \label{subsec:main-result_utility}

 At iteration $t\in\left[T\right]$ of Algorithm~\ref{algo:optimization-algo}, server averages the $ks$ received compressed and privatized gradients and obtains $\overline{\mathbf{g}}_t=\frac{1}{ks}\sum_{i\in\calU_t}\sum_{j\in\calS_{it}}\mathbf{q}_t(d_{ij})$ (line $12$ of Algorithm~\ref{algo:optimization-algo}) and then updates the parameter vector as $\theta_{t+1}\gets \prod_{\mathcal{C}}\left( \theta_t -\eta_t \overline{\mathbf{g}}_t\right)$. 
 Here, $\mathbf{q}_t(d_{ij})=\calR_p\(\nabla_{\theta_t} f(\theta_t;d_{ij})\)$.
 Since the randomized mechanism $\mathcal{R}_p$ is unbiased, the average gradient $\overline{\mathbf{g}}_t$ is also unbiased, i.e., we have $\mathbb{E}\left[\overline{\mathbf{g}}_t\right]=\nabla_{\theta_t}F\left(\theta_t\right)$, where expectation is taken with respect to the random sampling of clients and the data points as well as the randomness of the mechanism $\mathcal{R}_p$. Now we show that $\overline{\mathbf{g}}_t$ has a bounded second moment.
 \begin{lemma}\label{lem:2nd-moment-bound}
For any $d\in\frS$, if the function $f\left(\theta;.\right):\mathcal{C}\to \mathbb{R}$ is convex and $L$-Lipschitz continuous with respect to the $\ell_g$-norm, which is the dual of $\ell_p$-norm, then we have
\begin{align}
\mathbb{E}\|\overline{\mathbf{g}}_t\|_2^2 \leq L^2\max\lbrace d^{1-\frac{2}{p}},1\rbrace \( 1+ \frac{cd}{qn}\left(\frac{e^{\epsilon_0}+1}{e^{\epsilon_0}-1}\right)^2 \),
\end{align}
where $c$ is a global constant: $c=4$ if $p\in\{1,\infty\}$ and $c=14$ if $p\notin\{1,\infty\}$. 
\end{lemma}
\begin{proof}
Under the conditions of the lemma, we have from \cite[Lemma~$2.6$]{shalev2012online} that $\|\nabla_{\theta} f\left(\theta;d\right)\| \leq L$ for all $d\in\frS$, which implies that $\|\nabla_{\theta}F(\theta)\| \leq L$. 
Thus, we have
\begin{align*}
\mathbb{E}\|\overline{\mathbf{g}}_t\|_2^2&=\|\mathbb{E}\left[\overline{\mathbf{g}}_t\right]\|_2^2+\mathbb{E}\|\overline{\mathbf{g}}_t-\mathbb{E}\left[\overline{\mathbf{g}}_t\right]\|_2^2\\
&\stackrel{\left(a\right)}{\leq} \max\lbrace d^{1-\frac{2}{p}},1\rbrace L^2+\mathbb{E}\|\overline{\mathbf{g}}_t-\mathbb{E}\left[\overline{\mathbf{g}}_t\right]\|_2^2 \\
&\stackrel{\left(b\right)}{\leq}\max\lbrace d^{1-\frac{2}{p}},1\rbrace L^2+\frac{c L^2 \max\lbrace d^{2-\frac{2}{p}},d\rbrace}{ks}\left(\frac{e^{\epsilon_0}+1}{e^{\epsilon_0}-1}\right)^2\\
&\stackrel{\left(c\right)}{=}\max\lbrace d^{1-\frac{2}{p}},1\rbrace L^2+\frac{c L^2 \max\lbrace d^{2-\frac{2}{p}},d\rbrace}{qn}\left(\frac{e^{\epsilon_0}+1}{e^{\epsilon_0}-1}\right)^2,
\end{align*}
where $c$ is a global constant, and $c=4$ if $p\in\{1,\infty\}$ and $c=14$ if $p\notin\{1,\infty\}$. 
Step $\left(a\right)$ follows from the fact that $\|\nabla_{\theta_t}F\left(\theta_t\right)\| \leq L$ together with the norm inequality $\|\bu\|_q\leq\|\bu\|_p\leq d^{\frac{1}{p}-\frac{1}{q}}\|\bu\|_q$ for $1\leq p\leq q\leq\infty$. Step $\left(b\right)$ follows from Corollary~\ref{corol:L-p-norm_ub} with $\overline{p}=1$, i.e., for any $p$-norm, we use the mechanism for $\ell_2$-norm ball only (together with norm inequality) which gives the smallest variance. Step (c) uses $q=\frac{ks}{n}$.
\end{proof}
 
Now, we can use standard SGD convergence results for convex functions. In particular, we use the following result from \cite{shamir2013stochastic}.
\begin{lemma}[SGD Convergence~\cite{shamir2013stochastic}]\label{lem:convergence_sgd} 
Let $F\left(\theta\right)$ be a convex function, and the set $\mathcal{C}$ has diameter $D$. Consider a stochastic gradient descent algorithm $\theta_{t+1}\gets \prod_{\mathcal{C}}\left( \theta_t-\eta_t \mathbf{g}_t\right)$, where $\mathbf{g}_t$ satisfies $\mathbb{E}\left[\mathbf{g}_t\right]=\nabla_{\theta_t}F\left(\theta_t\right)$ and $\mathbb{E}\|\mathbf{g}_t\|_{2}^{2} \leq G^{2}$. By setting $\eta_t=\frac{D}{G\sqrt{t}}$, we get
\begin{equation}
\mathbb{E}\left[F\left(\theta_{T}\right)\right] - F\left(\theta^{*}\right) \leq 2DG\frac{2+\log\left(T\right)}{\sqrt{T}}=\mathcal{O}\left(DG\frac{\log\left(T\right)}{\sqrt{T}}\right).
\end{equation}
\end{lemma}
As shown in Lemma~\ref{lem:2nd-moment-bound} and above that Algorithm~\ref{algo:optimization-algo} satisfies the premise of Lemma~\ref{lem:convergence_sgd}. Now, using the bound on $G^2$ from Lemma~\ref{lem:2nd-moment-bound}, we have that the output $\theta_T$ of Algorithm~\ref{algo:optimization-algo} satisfies
\begin{align}\label{general_convergence-2}
\mathbb{E}\left[F\left(\theta_{T}\right)\right] - F\left(\theta^{*}\right) \leq \calO\(\frac{LD\log(T)\max\lbrace d^{\frac{1}{2}-\frac{1}{p}},1\rbrace}{\sqrt{T}} \( 1+ \sqrt{\frac{cd}{qn}}\left(\frac{e^{\epsilon_0}+1}{e^{\epsilon_0}-1}\right) \) \),
\end{align}
where we used the inequality $\sqrt{ 1+ \frac{cd}{qn}\left(\frac{e^{\epsilon_0}+1}{e^{\epsilon_0}-1}\right)^2 } \leq \( 1+ \sqrt{\frac{cd}{qn}}\left(\frac{e^{\epsilon_0}+1}{e^{\epsilon_0}-1}\right) \)$.

Note that if $\sqrt{\frac{cd}{qn}}\left(\frac{e^{\epsilon_0}+1}{e^{\epsilon_0}-1}\right) \leq \calO(1)$, then we recover the convergence rate of vanilla SGD without privacy. So, the interesting case is when $\sqrt{\frac{cd}{qn}}\left(\frac{e^{\epsilon_0}+1}{e^{\epsilon_0}-1}\right) \geq \Omega(1)$, which gives
\begin{equation*}
\mathbb{E}\left[F\left(\theta_{T}\right)\right] - F\left(\theta^{*}\right) \leq \calO\(\frac{LD\log(T)\max\lbrace d^{\frac{1}{2}-\frac{1}{p}},1\rbrace}{\sqrt{T}}\sqrt{\frac{cd}{qn}}\left(\frac{e^{\epsilon_0}+1}{e^{\epsilon_0}-1}\right)\).
\end{equation*}

This completes the proof of the third part of Theorem~\ref{thm:main-opt-result}.

\section{Discussion}\label{sec:Disc}

In this paper we have developed a compressed, private optimization
solution for a problem motivated by federated learning, where
distributed clients jointly build a common learning model. The main
technical contributions were developing order-optimal schemes for
private mean-estimation and combining them with privacy amplification
by sampling (of data and clients) as well as shuffling. We demonstrated
that iterative application of this enables us to get the same privacy,
optimization performance operating point as reported in
~\cite{erlingsson2020encode}, while obtaining order-wise improvement
in the number of bits required, per iteration, thereby getting these
communication gains for ``free''. Moreover, when the functions are
$L$-Lipschitz with respect to the $\ell_2$-norm, our scheme obtains
the optimal excess risk of the central differential privacy obtained
in~\cite{bassily2014private}, while operating in a distributed manner.

There are several open questions which are part of ongoing
investigations. These include sharper privacy analyses for these
schemes, which can improve the constants associated with the
performance parameters. It would also be important to extend these
ideas to non-convex functions and examine their numerical performance
for large-scale neural network models.

\section{Acknowledgment}
This work was supported in part by  a Google Faculty Research Award, NSF grant 1740047, and the UC-NL grant LFR-18-548554.
\bibliographystyle{alpha}
\bibliography{CompPrivRefs}

\appendix
\section{ Proof of Lemma~\ref{lemm:priv_opt}}~\label{app:supp_composition_priv}
Recall that the input dataset at client $i\in[m]$ is denoted by $\calD_i=\{d_{i1},d_{i2},\hdots,d_{ir}\}\in\frS^r$ and 
$\calD=\bigcup_{i=1}^m\calD_i$ denotes the entire dataset.
Recall from \eqref{mech-t} that the mechanism $\calM_t$ on input dataset $\calD$ can be defined as:
\begin{align}
\calM_t(\calD) = \calH_{ks} \circ \samp_{m,k} \(\calG_1,\ldots,\calG_m\),
\end{align}
where $\calG_i=\samp_{r,s} \left(\calR(\bx_{i1}^t),\hdots,\calR(\bx_{ir}^t)\right)$ and $\bx_{ij}^t=\nabla_{\theta_t} f(\theta_t;d_{ij}), \forall i\in[m],j\in[r]$. We define a mechanism $\calZ\left(\calD^{\left(t\right)}\right)=\calH_{ks} \left(\calR\left(\bx_{1}^t\right),\ldots,\calR\left(\bx_{ks}^t\right)\right)$ which is a shuffling of $ks$ outputs of local mechanism $\calR$, where $\calD^{\left(t\right)}$ denotes an arbitrary set of $ks$ data points and we index $\bx_i^t$'s from $i=1$ to $ks$ just for convenience. From the amplification by shuffling result \cite[Corollary~$5.3.1$]{balle2019privacy} (also see Lemma~\ref{lemma:amplification_shuffling}), the mechanism $\calZ$ is $(\tilde{\epsilon},\tilde{\delta})$-DP, where $\tilde{\delta}>0$ is arbitrary, and, if $\eps_0\leq \frac{\log\(ks/\log\left(1/\tilde{\delta}\right)\)}{2}$, then 
\begin{equation}\label{shuffling-priv-1}
\tilde{\epsilon} = \mathcal{O}\left(\min\{\epsilon_0,1\}e^{\epsilon_0}\sqrt{\frac{\log\left(1/\tilde{\delta}\right)}{ks}}\right).
\end{equation} 
Furthermore, when $\epsilon_0=\mathcal{O}\left(1\right)$, we get $\tilde{\epsilon}=\mathcal{O}\left(\epsilon_0\sqrt{\frac{\log\left(1/\tilde{\delta}\right)}{ks}}\right)$. 

Let $\calT\subseteq \lbrace 1,\ldots,m\rbrace$ denote the identities of the $k$ clients chosen at iteration $t$, and for $i\in\calT$, let $\calT_i\subseteq\lbrace 1,\ldots,r\rbrace$ denote the identities of the $s$ data points chosen at client $i$ at iteration $t$.\footnote{Though $\calT$ and $\calT_i, i\in\calT$ may be different at different iteration $t$, for notational convenience, we suppress the dependence on $t$ here.} For any $\calT\in\binom{[m]}{k}$ and $\calT_i\in\binom{[r]}{s},i\in\calT$, define $\overline{\calT}=\left(\calT,\calT_i,i\in\calT\right)$, $\calD^{\calT_i}=\{d_j : j\in\calT_i\}$ for $i\in\calT$, and $\calD^{\overline{\calT}}=\{\calD^{\calT_i} : i\in\calT\}$. Note that $\calT$ and $\calT_i, i\in\calT$ are random sets, where randomness is due to the sampling of clients and of data points, respectively. The mechanism $\calM_t$ can be equivalently written as $\calM_t=\calZ(\calD^{\overline{\calT}})$.

Observe that our sampling strategy is different from subsampling of choosing a uniformly random subset of $ks$ data points from the entire dataset $\calD$. Thus, we revisit the proof of privacy amplification by subsampling (see, for example, \cite{Jonathan2017sampling}) -- which is for uniform sampling -- to compute the privacy parameters of the mechanism $\calM_t$, where sampling is non-uniform. 
Define a dataset $\calD'=\left(\calD'_1\right)\bigcup\left(\cup_{i=2}^{m}\mathcal{D}_i\right)\in\frS^n$, where $\calD'_1=\lbrace d'_{11},d_{12},\ldots,d_{1r}\rbrace$ is different from the dataset $\calD_1$ in the first data point $d_{11}$. Note that $\calD$ and $\calD'$ are neighboring datasets -- where, we assume, without loss of generality, that the differing elements are $d_{11}$ and $d'_{11}$.

In order to show that $\calM_t$ is $(\epsbar,\delbar)$-DP, we need show that for an arbitrary subset $\calS$ of the range of $\calM_t$, we have
\begin{align}
\Pr\left[\calM_t\left(\calD\right)\in \calS\right] &\leq e^{\epsbar}\Pr\left[\calM_t\left(\calD'\right)\in \calS\right]  + \delbar \label{priv-comp-samp-dp1} \\
\Pr\left[\calM_t\left(\calD'\right)\in \calS\right] &\leq e^{\epsbar}\Pr\left[\calM_t\left(\calD\right)\in \calS\right]  + \delbar \label{priv-comp-samp-dp2}
\end{align}
Note that both \eqref{priv-comp-samp-dp1} and \eqref{priv-comp-samp-dp2} are symmetric, so it suffices to prove only one of them. We prove \eqref{priv-comp-samp-dp1} below.

Let $q=\frac{ks}{mr}$. We define conditional probabilities as follows:
\begin{align*}
A_{11}&=\Pr\left[\calZ(\calD^{\overline{\calT}})\in \calS |1\in\calT\ \text{and}\ 1\in\calT_1\right]\\
A'_{11}&=\Pr\left[\calZ(\calD^{'\overline{\calT}})\in \calS |1\in\calT\ \text{and}\ 1\in\calT_1\right]\\
A_{10}&=\Pr\left[\calZ(\calD^{\overline{\calT}})\in \calS |1\in\calT\ \text{and}\ 1\not\in\calT_1\right]=\Pr\left[\calZ(\calD^{'\overline{\calT}})\in \calS |1\in\calT\ \text{and}\ 1\not\in\calT_1\right]\\
A_{0}&=\Pr\left[\calZ(\calD^{\overline{\calT}})\in \calS |1\not\in\calT \right]=\Pr\left[\calZ(\calD^{'\overline{\calT}})\in \calS |1\not\in\calT\right]
\end{align*}
Let $q_1=\frac{k}{m}$ and $q_2=\frac{s}{r}$, and hence $q=q_1q_2$. Thus, we have
\begin{align*}
\Pr\left[\calM_t\left(\calD\right)\in \calS\right] &= qA_{11} + q_1\left(1-q_2\right)A_{10}+\left(1-q_1\right) A_{0}\\
\Pr\left[\calM_t\left(\calD'\right)\in \calS\right] &= qA'_{11} + q_1\left(1-q_2\right)A_{10}+\left(1-q_1\right) A_{0}
\end{align*}  
Note that the mechanism $\calZ$ is $(\tilde{\epsilon},\tilde{\delta})$-DP. 
Therefore, we have 
\begin{align}
A_{11} &\leq e^{\tilde{\epsilon}}A'_{11}+\tilde{\delta} \label{priv-comp-samp-1} \\
A_{11} &\leq e^{\tilde{\epsilon}}A_{10}+\tilde{\delta} \label{priv-comp-samp-2}
\end{align} 
Here \eqref{priv-comp-samp-1} is straightforward, but proving \eqref{priv-comp-samp-2} requires a combinatorial argument, which we give at the end of this proof.

We prove \eqref{priv-comp-samp-dp1} separately for two cases, first when $s=1$ and other when $s>1$; $k$ is arbitrary in both cases.
\subsection{For $s=1$ and arbitrary $k\in[m]$}
Since the mechanism $\calZ$ is $(\tilde{\epsilon},\tilde{\delta})$-DP, in addition to \eqref{priv-comp-samp-1}-\eqref{priv-comp-samp-2}, since $s=1$, we also have the following inequality:
\begin{align}
A_{11} &\leq e^{\tilde{\epsilon}}A_{0}+\tilde{\delta} \label{priv-comp-samp-3}
\end{align} 
Similar to \eqref{priv-comp-samp-2}, proving \eqref{priv-comp-samp-3} requires a combinatorial argument, which we will give at the end of this proof. Note that \eqref{priv-comp-samp-3} only holds for $s=1$ and may not hold for arbitrary $s$.

Inequalities \eqref{priv-comp-samp-1}-\eqref{priv-comp-samp-3} together imply $A_{11} \leq e^{\tilde{\eps}}\min\{A'_{11},A_{10},A_{0}\}+\tilde{\delta}$.
Now we prove \eqref{priv-comp-samp-dp1} for $\overline{\epsilon}=\ln(1+q(e^{\tilde{\epsilon}}-1)$ and $\overline{\delta}=q\tilde{\delta}$. Note that when $s=1$, we have $q_1=\frac{k}{m}$, $q_2=\frac{1}{r}$, and $q=\frac{k}{mr}$.
\begin{align*}
\Pr&\left[\calM_t\left(\calD\right)\in \calS\right] = qA_{11}+q_1\left(1-q_2\right)A_{10}+\left(1-q_1\right)A_{0} \\
&\leq q\(e^{\tilde{\epsilon}}\min\{A'_{11},A_{10},A_{0}\}+\tilde{\delta}\) +q_1\left(1-q_2\right)A_{10}+\left(1-q_1\right)A_{0}\\
&= q\left((e^{\tilde{\epsilon}}-1)\min\{A'_{11},A_{10},A_{0}\} + \min\{A'_{11},A_{10},A_{0}\}\right) + q_1\left(1-q_2\right)A_{10}+\left(1-q_1\right)A_{0} + q\tilde{\delta}\\
&\stackrel{\text{(a)}}{\leq} q(e^{\tilde{\epsilon}}-1)\min\{A'_{11},A_{10},A_{0}\} + qA'_{11} + q_1\left(1-q_2\right)A_{10}+\left(1-q_1\right)A_{0} + q\tilde{\delta}\\
&\stackrel{\text{(b)}}{\leq} q(e^{\tilde{\epsilon}}-1)\(qA'_{11}+q_1(1-q_2)A_{10}+(1-q_1)A_{0})\) + \(qA'_{11}+q_1\left(1-q_2\right)A_{10}+\left(1-q_1\right)A_{0}\) + q\tilde{\delta}\\
&=\left(1+q\left( e^{\tilde{\epsilon}}-1\right)\right)\left(qA'_{11}+q_1\left(1-q_2\right)A_{10}+\left(1-q_1\right)A_{0}\right) + q\tilde{\delta} \\
&= e^{\ln(1+q( e^{\tilde{\epsilon}}-1))}\Pr\left[\calM_t\left(\calD'\right)\in \calS\right]  + q\tilde{\delta}.
\end{align*}
Here, (a) follows from $\min\{A'_{11},A_{10},A_{0}\} \leq A'_{11}$, and (b) follows from the fact that minimum is upper-bounded by the convex combination.
By substituting the value of $\tilde{\eps}$ from \eqref{shuffling-priv-1} and using $ks=qn$, we get that for $\epsilon_0=\mathcal{O}\left(1\right)$, we have $\overline{\epsilon}=\mathcal{O}\left(\epsilon_0\sqrt{\frac{q\log\left(1/\tilde{\delta}\right)}{n}}\right)$.

\subsection{For $s>1$ and arbitrary $k\in[m]$}
Note that \eqref{priv-comp-samp-1}-\eqref{priv-comp-samp-2} together imply $A_{11} \leq e^{\tilde{\eps}}\min\{A'_{11},A_{10}\}+\tilde{\delta}$. Now we prove \eqref{priv-comp-samp-dp1} for $\overline{\epsilon}=\ln(1+q_2(e^{\tilde{\epsilon}}-1))$ and $\overline{\delta}=q\tilde{\delta}$.
\begin{align*}
\Pr&\left[\calM_t\left(\calD\right) \in \calS\right] =  qA_{11} + q_1(1-q_2) A_{10}+(1-q_1) A_{0} \\
&\leq q\(e^{\tilde{\epsilon}}\min\{A'_{11},A_{10}\} + \tilde{\delta}\) + q_1(1-q_2)A_{10}+(1-q_1)A_{0} \\
&= q\left((e^{\tilde{\epsilon}}-1)\min\{A'_{11},A_{10}\} + \min\{A'_{11},A_{10}\}\right)+ q_1(1-q_2)A_{10} + (1-q_1)A_{0}  +q\tilde{\delta}\\
&\stackrel{\text{(a)}}{\leq} q\(e^{\tilde{\epsilon}}-1)\min\{A'_{11},A_{10}\}\) + qA'_{11}+ q_1(1-q_2)A_{10} + (1-q_1)A_{0}  +q\tilde{\delta}\\
&\stackrel{\text{(b)}}{\leq} q\((e^{\tilde{\epsilon}}-1)(q_2A'_{11} + (1-q_2)A_{10})\) + \(qA'_{11}+ q_1(1-q_2)A_{10} + (1-q_1)A_{0}\)  +q\tilde{\delta}\\
&= q_2\((e^{\tilde{\epsilon}}-1)(q_1q_2A'_{11} + q_1(1-q_2)A_{10})\) + \(qA'_{11}+ q_1(1-q_2)A_{10} + (1-q_1)A_{0}\)  +q\tilde{\delta}\\
&\stackrel{\text{(c)}}{\leq} q_2\((e^{\tilde{\epsilon}}-1)(qA'_{11} + q_1(1-q_2)A_{10}) + (1-q_1)A_0\) + \(qA'_{11}+ q_1(1-q_2)A_{10} + (1-q_1)A_{0}\)  +q\tilde{\delta}\\
&= \(1+q_2\((e^{\tilde{\epsilon}}-1)\)(qA'_{11} + q_1(1-q_2)A_{10}) + (1-q_1)A_0\) + q\tilde{\delta}\\
&= e^{\ln(1+q_2( e^{\tilde{\epsilon}}-1))}\Pr\left[\calM_t\left(\calD'\right)\in \calS\right]  + q\tilde{\delta}
\end{align*}
Here, (a) follows from $\min\{A'_{11},A_{10}\} \leq A'_{11}$, (b) follows from the fact that minimum is upper-bounded by the convex combination, and (c) holds because $(1-q_1)A_0\geq0$. By substituting the value of $\tilde{\eps}$ from \eqref{shuffling-priv-1} and using $ks=qn$, we get that for $\epsilon_0=\mathcal{O}\left(1\right)$, we have $\overline{\epsilon}=\mathcal{O}\left(\epsilon_0\sqrt{\frac{q_2\log\left(1/\tilde{\delta}\right)}{q_1n}}\right)$. Note that when $q_1=1$ (i.e., we select all the clients in each iteration), then this gives the desired privacy amplification of $q=q_2$.

%
%
%
%
%
%
%
%

The proof of Lemma~\ref{lemm:priv_opt} is complete, except for that we have to prove \eqref{priv-comp-samp-2} and \eqref{priv-comp-samp-3}. Before proving \eqref{priv-comp-samp-2} and \eqref{priv-comp-samp-3}, we state an important remark about the privacy amplification in both the cases.

\begin{remark}
Note that when $s=1$ and $\eps_0=\calO(1)$, we have $\epsbar = \ln(1+q( e^{\tilde{\epsilon}}-1)) = \calO(q\tilde{\epsilon})$. So we get a privacy amplification by a factor of $q=\frac{ks}{mr}$ -- the sampling probability of each data point from the entire dataset. Here, we get a privacy amplification from both types of sampling, of clients as well of data points.

On the other hand, when $s>1$ and $\eps_0=\calO(1)$, we have $\epsbar = \ln(1+q_2( e^{\tilde{\epsilon}}-1)) = \calO(q_2\tilde{\epsilon})$, which, unlike the case of $s=1$, only gives the privacy amplification by a factor of $q_2=\frac{s}{r}$ -- the sampling probability of each data point from a client. So, unlike the case of $s=1$, here we only get a privacy amplification from sampling of data points, not from sampling of clients. 
Note that when $k=m$ and any $s\in[r]$ (which implies $q_1=1$ and $q=q_2$), we have $\overline{\epsilon}=\mathcal{O}\left(\epsilon_0\sqrt{\frac{q_2\log\left(1/\tilde{\delta}\right)}{n}}\right)$, which gives the desired amplification when we select all the clients in each iteration.
\end{remark}

\paragraph{Proof of \eqref{priv-comp-samp-2}.}
First note that the number of subsets $\calT_1\subset[r]$ such that $|\calT_1|=s,1\in\calT_1$ is equal to $\binom{r-1}{s-1}$ and the number of subsets $\calT_1\subset[r]$ such that $|\calT_1|=s,1\notin\calT_1$ is equal to $\binom{r-1}{s}$. It is easy to verify that $(r-s)\binom{r-1}{s-1}=s\binom{r-1}{s}$.

Consider the following bipartite graph $G=(V_1\cup V_2,E)$, where the left vertex set $V_1$ has $\binom{r-1}{s-1}$ vertices, one for each configuration of $\calT_1\subset[r]$ such that $|\calT_1|=s,1\in\calT_1$, the right vertex set $V_2$ has $\binom{r-1}{s}$ vertices, one for each configuration of $\calT_1\subset[r]$ such that $|\calT_1|=s,1\notin\calT_1$, and the edge set $E$ contains all the edges between neighboring vertices, i.e., if $(\bu,\bv)\in V_1\times V_2$ is such that $\bu$ and $\bv$ differ in only one element, then $(\bu,\bv)\in E$. 
Observe that each vertex of $V_1$ has $(r-s)$ neighbors in $V_2$ -- the neighbors of $\calT_1\in V_1$ will be $\{(\calT_1\setminus\{1\})\cup\{i\} : i\in[m]\setminus\calT_1\}\subset V_2$. Similarly, each vertex of $V_2$ has $s$ neighbors in $V_1$ -- the neighbors of $\calT_1\in V_2$ will be $\{(\calT_1\setminus\{i\})\cup\{1\} : i\in\calT_1\} \subset V_1$.

Now, fix any $\calT\in\binom{[m]}{k}$ s.t.\ $1\in\calT$, and for $i\in\calT\setminus\{1\}$, fix any $\calT_i\in\binom{[r]}{s}$, and consider an arbitrary $(\bu,\bv)\in E$. Since the mechanism $\calZ$ is $(\tilde{\epsilon},\tilde{\delta})$-DP, we have
\begin{equation}\label{bipartite-nbrs}
\Pr\left[\calZ(\calD^{\overline{\calT}})\in \calS |1\in\calT,\calT_1=\bu,\calT_i, i\in\calT\setminus\{1\}\right] \leq e^{\tilde{\eps}}\Pr\left[\calZ(\calD^{\overline{\calT}})\in \calS |1\in\calT,\calT_1=\bv,\calT_i, i\in\calT\setminus\{1\}\right] + \tilde{\delta}.
\end{equation} 
Now we are ready to prove \eqref{priv-comp-samp-2}.
\begin{align*}
A_{11} &= \Pr\left[\calZ(\calD^{\overline{\calT}})\in \calS |1\in\calT\ \text{and}\ 1\in\calT_1\right] \\
&= \sum_{\substack{\calT\in\binom{[m]}{k}: 1\in\calT \\ \calT_1\in\binom{[r]}{s}: 1\in\calT_1 \\ \calT_i\in\binom{[r]}{s}\text{ for }i\in\calT\setminus\{1\}}} \Pr[\calT,\calT_i,i\in\calT | 1\in\calT\ \text{and}\ 1\in\calT_1] \Pr[\calZ(\calD^{\overline{\calT}})\in\calS | \calT,\calT_1,\hdots,\calT_m] \\
&\stackrel{\text{(a)}}{=} \sum_{\substack{\calT\in\binom{[m]}{k}: 1\in\calT \\ \calT_i\in\binom{[r]}{s}\text{ for }i\in\calT\setminus\{1\}}} \Pr[\calT,\calT_i,i\in\calT\setminus\{1\} | 1\in\calT ] \sum_{\calT_1\in\binom{[r]}{s}: 1\in\calT_1}\Pr[\calT_1 | 1\in\calT_1] \Pr[\calZ(\calD^{\overline{\calT}})\in\calS | \calT,\calT_1,\hdots,\calT_m] \\
&= \sum_{\substack{\calT\in\binom{[m]}{k}: 1\in\calT \\ \calT_i\in\binom{[r]}{s}\text{ for }i\in\calT\setminus\{1\}}} \Pr[\calT,\calT_i,i\in\calT\setminus\{1\} | 1\in\calT ] \frac{1}{(r-s)\binom{r-1}{s-1}} \sum_{\calT_1\in\binom{[r]}{s}: 1\in\calT_1} (r-s)\Pr[\calZ(\calD^{\overline{\calT}})\in\calS | \calT,\calT_1,\hdots,\calT_m] \\
&= \sum_{\substack{\calT\in\binom{[m]}{k}: 1\in\calT \\ \calT_i\in\binom{[r]}{s}\text{ for }i\in\calT\setminus\{1\}}} \Pr[\calT,\calT_i,i\in\calT\setminus\{1\} | 1\in\calT ] \frac{1}{s\binom{r-1}{s}} \sum_{\calT_1\in\binom{[r]}{s}: 1\in\calT_1} (r-s)\Pr[\calZ(\calD^{\overline{\calT}})\in\calS | \calT,\calT_1,\hdots,\calT_m] \\
&\stackrel{\text{(b)}}{\leq} \sum_{\substack{\calT\in\binom{[m]}{k}: 1\in\calT \\ \calT_i\in\binom{[r]}{s}\text{ for }i\in\calT\setminus\{1\}}} \Pr[\calT,\calT_i,i\in\calT\setminus\{1\} | 1\in\calT ] \frac{1}{s\binom{r-1}{s}} \sum_{\calT_1\in\binom{[r]}{s}: 1\notin\calT_1} s\(e^{\tilde{\eps}}\Pr[\calZ(\calD^{\overline{\calT}})\in\calS | \calT,\calT_1,\hdots,\calT_m] + \tilde{\delta}\) \\
&= \sum_{\substack{\calT\in\binom{[m]}{k}: 1\in\calT \\ \calT_i\in\binom{[r]}{s}\text{ for }i\in\calT\setminus\{1\}}} \Pr[\calT,\calT_i,i\in\calT\setminus\{1\} | 1\in\calT ] \sum_{\calT_1\in\binom{[r]}{s}: 1\notin\calT_1} \Pr[\calT_1 | 1\notin\calT_1]\(e^{\tilde{\eps}}\Pr[\calZ(\calD^{\overline{\calT}})\in\calS | \calT,\calT_1,\hdots,\calT_m] + \tilde{\delta}\) \\
&\stackrel{\text{(c)}}{=} \sum_{\substack{\calT\in\binom{[m]}{k}: 1\in\calT \\ \calT_1\in\binom{[r]}{s}: 1\notin\calT_1 \\ \calT_i\in\binom{[r]}{s}\text{ for }i\in\calT\setminus\{1\}}} \Pr[\calT,\calT_i,i\in\calT | 1\in\calT\ \text{and}\ 1\notin\calT_1] \(e^{\tilde{\eps}}\Pr[\calZ(\calD^{\overline{\calT}})\in\calS | \calT,\calT_1,\hdots,\calT_m] + \tilde{\delta}\) \\
&\leq e^{\tilde{\eps}}\Pr\left[\calZ(\calD^{\overline{\calT}})\in \calS |1\in\calT\ \text{and}\ 1\notin\calT_1\right] + \tilde{\delta} \\
&= e^{\tilde{\eps}}A_{10} + \tilde{\delta}.
\end{align*}
Here, (a) and (c) follow from the fact that clients sample the data points independent of each other, and (b) follows from \eqref{bipartite-nbrs} together with the fact that there are $(r-s)\binom{r-1}{s-1}=s\binom{r-1}{s}$ edges in the bipartite graph $G=(V_1\cup V_2, E)$, where degree of vertices in $V_1$ is $(r-s)$ and degree of vertices in $V_2$ is $s$.

\paragraph{Proof of \eqref{priv-comp-samp-3}.}
First note that the number of subsets $\calT\in[m]$ such that $|\calT|=k,1\in\calT$ is equal to $\binom{m-1}{k-1}$ and the number of subsets $\calT\subset[m]$ such that $|\calT|=k,1\notin\calT$ is equal to $\binom{m-1}{k}$. It is easy to verify that $(m-k)\binom{m-1}{k-1}=k\binom{m-1}{k}$.

Consider the following bipartite graph $G=(V_1\cup V_2,E)$, where the left vertex set $V_1$ has $\binom{m-1}{k-1}r^{k-1}$ vertices, one for each configuration of $\left(\calT,\calT_i : i\in\calT\right)$ such that $\calT\subset[m]$, $|\calT|=k,1\in\calT$ and $\calT_1=1$, the right vertex set $V_2$ has $\binom{m-1}{k}r^{k}$ vertices, one for each configuration of $\left(\calT,\calT_i: i\in\calT\right)$ such that $\calT\subset[m]$, $|\calT|=k,1\notin\calT$, and the edge set $E$ contains all the edges between neighboring vertices, i.e., if $(\bu,\bv)\in V_1\times V_2$ is such that $\bu$ and $\bv$ differ in only one element, then $(\bu,\bv)\in E$. 
Observe that each vertex of $V_1$ has $r(m-k)$ neighbors in $V_2$. Similarly, each vertex of $V_2$ has $k$ neighbors in $V_1$.

Consider an arbitrary edge $(\bu,\bv)\in E$. By construction, there exists $\calT\in\binom{[m]}{k}$ with $1\in\calT$ and $\calT_{i}\in[r],i\in\calT$ such that $\bu=(\calT,\calT_i:i\in\calT)$ and $\calT'\in\binom{[m]}{k}$ with $1\notin\calT'$ and $\calT'_{i}\in[r],i\in\calT'$ such that $\bv=(\calT',\calT'_i:i\in\calT')$. Note that, since $(\bu,\bv)\in E$, $(\calT_i:i\in\calT)$ and $(\calT'_i:i\in\calT')$ have $k-1$ elements common. Now, since the mechanism $\calZ$ is $(\tilde{\epsilon},\tilde{\delta})$-DP, we have
\begin{equation}\label{bipartite-nbrs-3}
\Pr\left[\calZ(\calD^{\overline{\calT}})\in \calS | \calT, \calT_i,i\in\calT\right] \leq e^{\tilde{\eps}}\Pr\left[\calZ(\calD^{\overline{\calT'}})\in \calS | \calT', \calT'_i, i\in\calT'\right] + \tilde{\delta}.
\end{equation} 
Now we are ready to prove \eqref{priv-comp-samp-3}.
\begin{align*}
A_{11} &= \Pr\left[\calZ(\calD^{\overline{\calT}})\in \calS |1\in\calT\ \text{and}\ \calT_1=1\right] \\
&= \sum_{\substack{\calT\in\binom{[m]}{k}: 1\in\calT \\ \calT_{i}\in [r] \text{ for } i\in\calT : \calT_1=1}} \Pr[\calT,\calT_i,i\in\calT |1\in\calT\ \text{and}\ \calT_1=1] \Pr[\calZ(\calD^{\overline{\calT}})\in\calS | \calT,\calT_i,i\in\calT] \\
&=\frac{1}{\binom{m-1}{k-1}r^{k-1}}\sum_{\substack{\calT\in\binom{[m]}{k}: 1\in\calT \\ \calT_{i}\in [r] \text{ for } i\in\calT : \calT_1=1}} \Pr[\calZ(\calD^{\overline{\calT}})\in\calS | \calT,\calT_i,i\in\calT]\\
&= \frac{1}{(m-k)\binom{m-1}{k-1}r^k}\sum_{\substack{\calT\in\binom{[m]}{k}: 1\in\calT \\ \calT_{i}\in [r] \text{ for } i\in\calT : \calT_1=1}} r(m-k)\Pr[\calZ(\calD^{\overline{\calT}})\in\calS | \calT,\calT_i,i\in\calT]\\
&\stackrel{\text{(a)}}{=} \frac{1}{k\binom{m-1}{k}r^{k}}\sum_{\substack{\calT\in\binom{[m]}{k}: 1\in\calT \\ \calT_{i}\in [r] \text{ for } i\in\calT : \calT_1=1}} r(m-k) \Pr[\calZ(\calD^{\overline{\calT}})\in\calS | \calT,\calT_i,i\in\calT] \\
&\stackrel{\text{(b)}}{\leq}\frac{1}{k\binom{m-1}{k}r^{k}}\sum_{\substack{\calT\in\binom{[m]}{k} : 1\notin\calT \\ \calT_{i}\in[r] \text{ for } i\in\calT}} k\left(e^{\epsilon} \Pr[\calZ(\calD^{\overline{\calT}})\in\calS | \calT,\calT_i,i\in\calT]+\tilde{\delta} \right)\\
&=\frac{1}{\binom{m-1}{k}r^{k}}\sum_{\substack{\calT\in\binom{[m]}{k}: 1\notin\calT \\ \calT_{i}\in[r]\text{ for } i\in\calT}} \left(e^{\epsilon} \Pr[\calZ(\calD^{\overline{\calT}})\in\calS | \calT,\calT_i,i\in\calT]+\tilde{\delta}\right)\\
&=\sum_{\substack{\calT\in\binom{[m]}{k}: 1\notin\calT \\ \calT_{i}\in[r]\text{ for } i\in\calT}} \Pr[\calT,\calT_i,i\in\calT | 1\notin\calT] \left(e^{\epsilon} \Pr[\calZ(\calD^{\overline{\calT}})\in\calS | \calT,\calT_i,i\in\calT]+\tilde{\delta}\right)\\
&= e^{\tilde{\eps}} \Pr\left[\calZ(\calD^{\overline{\calT}})\in \calS |1\notin\calT \right] + \tilde{\delta} \\
&=e^{\tilde{\eps}} A_0 + \tilde{\delta}
\end{align*}
Here, (a) uses $(m-k)\binom{m-1}{k-1}=k\binom{m-1}{k}$, and (b) follows from \eqref{bipartite-nbrs-3} together with the fact that there are $r(m-k)\binom{m-1}{k-1}r^{k-1}=k\binom{m-1}{k}r^{k}$ edges in the bipartite graph $G=(V_1\cup V_2, E)$, where degree of vertices in $V_1$ is $r(m-k)$ and degree of vertices in $V_2$ is $k$. \\

\noindent This completes the proof of Lemma~\ref{lemm:priv_opt}.

\section{Minimax Risk Estimation}\label{app:minimax}

\begin{lemma} \label{lem:deterministic-estimator}
For the minimax problems~\eqref{eqn1_1} and~\eqref{eqn1_2}, the optimal estimator $\hatx\left(\by^{n}\right)$ is a deterministic function. In other words, the randomized decoder does not help  in reducing the minimax risk. 
\end{lemma}
\begin{proof}
Towards a contradiction, suppose that the optimal estimator $\hatx$ is a randomized decoder defined as follows. For given clients' responses $\by^{n}$, let the probabilistic estimator generate an estimate $\hatx\left(\by^{n}\right)$ whose mean and trace of the covariance matrix are given by $\bmu_{\hatx(\by^n)}=\mathbb{E}\left[\hatx(\by^n)\right]$ and $\sigma_{\hatx(\by^n)}^{2}=\mathbb{E}\left[\left\|\hatx\left(\by^{n}\right)-\mu_{\hatx(\by^{n})}\right\|_{2}^{2}\big|Y^{n}\right]$, respectively, where expectation is taken with respect to the randomization of the decoder, conditioned of $Y^n$.
\begin{align*}
\mathbb{E}\left[\left\|\overline{\bx}-\hatx\left(\by^{n}\right)\right\|_{2}^{2}\big|\by^{n}\right]&=\mathbb{E}\left[\left\|\overline{\bx}-\bmu_{\hatx(\by^{n})}+\bmu_{\hatx(\by^{n})}-\hatx\left(\by^{n}\right)\right\|_{2}^{2}\big|\by^{n}\right]\\
&=\mathbb{E}\left[\left\|\overline{\bx}-\bmu_{\hatx(\by^{n})}\right\|_{2}^{2}\big|\by^{n}\right]+\mathbb{E}\left[\left\|\bmu_{\hatx(\by^{n})}-\hatx\left(\by^{n}\right)\right\|_{2}^{2}\big|\by^{n}\right]\\
&\hspace{1cm} + 2\mathbb{E}\left\langle\overline{\bx}-\bmu_{\hatx(\by^{n})}, \bmu_{\hatx(\by^{n})}-\hatx\left(\by^{n}\right)\big|\by^{n}\right\rangle\\
&\stackrel{\text{(a)}}{=} \mathbb{E}\left[\left\|\overline{\bx}-\bmu_{\hatx(\by^{n})}\right\|_{2}^{2}\big|\by^{n}\right]+\sigma_{\hatx(\by^{n})}^{2} \\
&>\mathbb{E}\left[\left\|\overline{\bx}-\bmu_{\hatx(\by^{n})}\right\|_{2}^{2}\big|\by^{n}\right]
\end{align*}
In (a), we used that $\bmu_{\hatx(\by^n)}=\mathbb{E}\left[\hatx(\by^n)\right]$ to eliminate the last term. Similarly, we can prove that $\mathbb{E}\left[\|\bmu_{\textbf{q}}-\hatx\left(\by^{n}\right)\|_{2}^{2}\big|\by^{n}\right]> \mathbb{E}\left[\|\bmu_{\textbf{q}}-\bmu_{\by^{n}}\|_{2}^{2}\big|\by^{n}\right]$. Hence, the deterministic estimator $\hatx\left(\by^{n}\right)=\mu_{\hatx(\by^{n})}$ has a lower minimax risk than the probabilistic estimator. 
\end{proof}

\section{Compressed and Private Mean Estimation}\label{app:PrivMeanEst}

\subsection{Achievability for $\ell_1$-norm Ball: Proof of Theorem~\ref{thm:L1-norm_ub}}\label{app:L1-norm-algo1}

\begin{lemma*}[Restating Lemma~\ref{R1-quantize-properties}]
The mechanism $\calR_1$ presented in Algorithm~\ref{R1-quantize-client} satisfies the following properties:
\begin{enumerate}
\item $\calR_1$ is $\left(\eps_0,\log\left(d\right)+1\right)$-LDP and requires only $1$-bit of communication using public-randomness. 
\item $\calR_1$ is unbiased and has bounded variance, i.e., for every $\bx\in\calB_1^d\left(a\right)$, we have
$$\bbE\left[\calR_1\left(\bx\right)\right]=\bx \quad \text{and} \quad \bbE\|\calR_1\left(\bx\right)-\bx\|_{2}^{2}\leq d\left(\frac{e^{\eps_0}+1}{e^{\eps_0}-1}\right)^{2}.$$
\end{enumerate}
\end{lemma*}
\begin{proof}
We show these properties one-by-one below.
\begin{enumerate}
\item Observe that the output of the mechanism $\calR_1$ can be represented using the index $j\in\left[d\right]$ and one bit of the sign of $\lbrace \pm a\bH_d\left(j\right)\left(\frac{e^{\eps_0}+1}{e^{\eps_0}-1}\right) \rbrace$. Hence, it requires only $\log\left(d\right)+1$ bits for communication. Furthermore, the randomness  $j\sim \textsf{Unif}\left[d\right]$ is independent of the input $\bx$. Thus, if the client has access to a public randomness $j$, then the client needs only to send one bit to represent its sign. Now, we show that the mechanism $\calR_1$ is $\eps_0$-LDP. Let $\calZ = \big\lbrace \pm a\bH_d(j) \big(\frac{e^{\eps_0}+1}{e^{\eps_0}-1}\big) : j=1,2,\hdots,d \big\rbrace$ denote all possible $2d$ outputs of the mechanism $\calR_1$. We get 
\begin{align*}
\sup_{\bx,\bx'\in\calB_{1}^{d}\left(a\right)}\sup_{\bz\in\calZ}\frac{\Pr[\calR_1(\bx)=\bz]}{\Pr[\calR_1(\bx')=\bz]} &\leq \sup_{\bx,\bx^{\prime}\in\calB_{1}^{d}\left(a\right)} \frac{\frac{1}{d}\sum_{j=1}^d \(\frac{1}{2} + \frac{\sqrt{d}|y_j|}{2a}\frac{e^{\eps_0}-1}{e^{\eps_0}+1}\)}{\frac{1}{d}\sum_{j=1}^d \(\frac{1}{2} - \frac{\sqrt{d}|y'_j|}{2a}\frac{e^{\eps_0}-1}{e^{\eps_0}+1}\)} \\
&= \sup_{\bx,\bx^{\prime}\in\calB_{1}^{d}\left(a\right)} \frac{\frac{1}{d}\sum_{j=1}^d \(a(e^{\eps_0}+1) + \sqrt{d}|y_j|(e^{\eps_0}-1)\)}{\frac{1}{d}\sum_{j=1}^d \(a(e^{\eps_0}+1) - \sqrt{d}|y'_j|(e^{\eps_0}-1)\)} \\
&\stackrel{\text{(a)}}{\leq} \frac{2 a e^{\eps_0}}{2a}=e^{\eps_0},
\end{align*}
where (a) uses the fact that for every $j\in[d]$, we have $|y_j| \leq \nicefrac{a}{\sqrt{d}}$ and $|y'_j| \leq \nicefrac{a}{\sqrt{d}}$.

\item Fix an arbitrary $\bx\in\calB_{1}^{d}\left(a\right)$.
\begin{align*}
\text{Unbiasedness:} \quad \bbE\left[\calR_1\left(\bx\right)\right]&=\frac{1}{d}\sum_{j=1}^{d}a\bH_d\left(j\right)  \left(\frac{e^{\eps_0}+1}{e^{\eps_0}-1}\right)\(\frac{\sqrt{d}y_j}{a} \frac{e^{\eps_0}-1}{e^{\eps_0}+1}\) \\
&=\frac{1}{d}\sum_{j=1}^{d}\bH_d\left(j\right)\sqrt{d}y_j 
\ \stackrel{\text{(b)}}{=}\frac{1}{d}\sum_{j=1}^{d}\bH_d\left(j\right)\bH_d^{T}(j)\bx 
\ \stackrel{\text{(c)}}{=} \bx 
\end{align*}
where (b) uses $\by=\frac{1}{\sqrt{d}}\bH_d\bx$ and (c) uses $\sum_{j=1}^d \bH_d(j)\bH_d^{T}(j) = \bH_d\bH_d^T = d\bI_d$.

\begin{align*}
\text{Bounded variance:} \quad \bbE\|\calR_1\left(\bx\right)-\bx\|_{2}^{2} &\leq \bbE\|\calR_1(\bx)\|^2 = \bbE[\calR_1(\bx)^T\calR_1(\bx)] \\
&= \frac{1}{d}\sum_{j=1}^d a^2\bH_d(j)^T\bH_d(j)\(\frac{e^{\eps_0}+1}{e^{\eps_0}-1}\)^2 \\
&= a^2 d\left(\frac{e^{\eps_0}+1}{e^{\eps_0}-1}\right)^{2} \tag{Since $\bH_d(j)^T\bH_d(j)=d, \forall j\in[d]$}
\end{align*}
\end{enumerate}
This completes the proof of Lemma~\ref{R1-quantize-properties}.
\end{proof}

\subsection{Achievability for $\ell_2$-norm Ball: Proof of Theorem~\ref{thm:L2-norm_ub}}\label{app:L-infty-norm_ub_proof}
\begin{lemma*}[Restating Lemma~\ref{R_2-quantize-properties}]
The mechanism $\calR_2$ presented in Algorithm~\ref{R_2-quantize-client} satisfies the following properties, where $\eps_0>0$:
\begin{enumerate}
\item $\calR_2$ is $\left(\eps_0,d(\log(e)+1)\right)$-LDP.
\item $\calR_2$ is unbiased and has bounded variance, i.e., for every $\bx\in\calB_2^d\left(a\right)$, we have 
$$\bbE\left[\calR_2\left(\bx\right)\right]=\bx \quad \text{and} \quad \bbE\|\calR_2\left(\bx\right)-\bx\|_{2}^{2}\leq 6a^2 d\(\frac{e^{\eps_0}+1}{e^{\eps_0}-1}\)^{2}.$$
\end{enumerate}
\end{lemma*}
\begin{proof} 
We prove these properties one-by-one below.
\begin{enumerate}
\item It was shown by Duchi et al.~\cite[Section~$4.2.3$]{duchi2018minimax} that \textsf{Priv} is an $\eps_0$-LDP mechanism.
Now, since $\calR_2 = \textsf{Quan}\circ\textsf{Priv}$ is a post-processing of a differentially-private mechanism \textsf{Priv} and post-processing preserves differential privacy, we have that $\calR_2$ is also $\eps_0$-LDP.
The claim that $\calR_2$ uses $d(\log(e)+1)$ bits of communication follows because $\calR_2$ outputs the result of \textsf{Quan}, which produces an output which can be represented using $d(\log(e)+1)$ bits; see \cite{mayekar2020limits}.

\item Unbiasedness of $\calR_2$ follows because $\calR_2 = \textsf{Quan}\circ\textsf{Priv}$ and both \textsf{Priv} and \textsf{Quan} are unbiased. To prove that variance is bounded, fix an $\bx\in\calB_2^d\left(a\right)$.
\begin{align*}
\bbE\|\calR_2(\bx)-\bx\|_{2}^{2} &= \bbE\|\textsf{Quan}\(\textsf{Priv}(\bx)\)-\bx\|_{2}^{2} \\
&= \bbE\|\textsf{Quan}\(\textsf{Priv}(\bx)\) - \textsf{Priv}(\bx) + \textsf{Priv}(\bx) - \bx\|_{2}^{2} \\
&\stackrel{\text{(a)}}{=} \bbE\|\textsf{Quan}\(\textsf{Priv}(\bx)\)-\textsf{Priv}(\bx)\|_{2}^{2} + \bbE\|\textsf{Priv}(\bx)-\bx\|_{2}^{2} \\
&\stackrel{\text{(b)}}{\leq} 2\|\textsf{Priv}(\bx)\|^2 + \bbE\|\textsf{Priv}(\bx)\|^2 \\
&\stackrel{\text{(c)}}{\leq} 3\|\textsf{Priv}(\bx)\|^2 \stackrel{\text{(d)}}{\leq} 6d\(\frac{e^{\eps_0}+1}{e^{\eps_0}-1}\)^{2}.
\end{align*}
In (a) we used the fact that $\textsf{Quan}$ and $\textsf{Priv}$ are unbiased, which implies that the cross multiplication term is zero. 
In (b) we used Lemma~\ref{l-2-quan-properties} to write $\bbE\|\textsf{Quan}\(\textsf{Priv}(\bx)\)-\textsf{Priv}(\bx)\|_{2}^{2} \leq 2\|\textsf{Priv}(\bx)\|^2$ and used the unbiasedness of \textsf{Priv} together with the fact that variance is bounded by the second moment to write $\bbE\|\textsf{Priv}(\bx)-\bx\|_{2}^{2}\leq\bbE\|\textsf{Priv}(\bx)\|_{2}^{2}$. In (c) we used that the length of \textsf{Priv} on any input remains fixed, i.e., $\bbE\|\textsf{Priv}(\bx)\|^2=\|\textsf{Priv}(\bx)\|^2 = M^2$ (where $M$ is from the line 4 of Algorithm~\ref{l-2-privacy}) holds for any $\bx\in\calB_2^d(a)$. In (d) we used the bound on $\|\textsf{Priv}(\bx)\|_2^2$ from Lemma~\ref{l-2-privacy-properties}.
\end{enumerate}
This completes the proof of Lemma~\ref{R_2-quantize-properties}.
\end{proof}

\subsection{Achievability for $\ell_{\infty}$-norm Ball: Proof of Theorem~\ref{thm:L-infty-norm_ub}}\label{app:L-infty-norm-algo}

\begin{lemma*}[Restating Lemma~\ref{R_infty-quantize-properties}]
The mechanism $\calR_\infty$ presented in Algorithm~\ref{R_infty-quantize-client} satisfies the following properties:
\begin{enumerate}
\item $\calR_\infty$ is $\left(\eps_0,\log\left(d\right)+1\right)$-LDP and requires only $1$-bit of communication using public-randomness. 
\item $\calR_{\infty}$ is unbiased and has bounded variance, i.e., for every $\bx\in\calB_{\infty}^d\left(a\right)$, we have
$$\bbE\left[\calR_\infty\left(\bx\right)\right]=\bx \quad \text{and} \quad \bbE\|\calR_\infty\left(\bx\right)-\bx\|_{2}^{2}\leq a^2 d^{2}\left(\frac{e^{\eps_0}+1}{e^{\eps_0}-1}\right)^{2}.$$
\end{enumerate}
\end{lemma*}
\begin{proof}
We prove these properties one-by-one below.
\begin{enumerate}
\item Observe that the output of the mechanism $\calR_\infty$ can be represented using the index $j\in\left[d\right]$ and one bit for the sign of $\big\lbrace \pm ad\big(\frac{e^{\eps_0}+1}{e^{\eps_0}-1}\big)\be_j \big\rbrace$. Hence, it requires only $\log\left(d\right)+1$ bits for communication. Furthermore, the randomness  $j\sim \textsf{Unif}\left[d\right]$ is independent of the input $\bx$. Thus, if the client has access to a public randomness $j$, then the client needs only to send one bit for its sign. Now, we show that the mechanism $\calR_\infty$ is $\eps_0$-LDP. Let $\calZ =  \big\lbrace \pm ad\big(\frac{e^{\eps_0}+1}{e^{\eps_0}-1}\big)\be_j : j=1,2,\hdots,d \big\rbrace$ denote all possible $2d$ outputs of the mechanism $\calR_\infty$. We get
\begin{align}
\sup_{\bx,\bx^{\prime}\in\calB_{\infty}^{d}\left(a\right)}\sup_{\bz\in\calZ}\frac{\Pr\left[\calR_\infty\left(\bx\right)=\bz\right]}{\Pr\left[\calR_\infty\left(\bx\right)=\bz\right]}&\leq \sup_{\bx,\bx^{\prime}\in\calB_{\infty}^{d}\left(a\right)} 
\frac{\frac{1}{d}\sum_{i=1}^d\(\frac{1}{2}+\frac{|x_j|}{2a}\frac{e^{\eps_0}-1}{e^{\eps_0}+1}\)}{\frac{1}{d}\sum_{i=1}^d\(\frac{1}{2}-\frac{|x'_j|}{2a}\frac{e^{\eps_0}-1}{e^{\eps_0}+1}\)}\\
&= \sup_{\bx,\bx^{\prime}\in\calB_{\infty}^{d}} \frac{\frac{1}{d}\sum_{i=1}^d\(a(e^{\eps_0}+1)+|x_j|(e^{\eps_0}-1)\)}{\frac{1}{d}\sum_{i=1}^d\(a(e^{\eps_0}+1)-|x'_j|(e^{\eps_0}-1)\)}\\
&\stackrel{\text{(a)}}{\leq} \frac{2ae^{\eps_0}}{2a}=e^{\eps_0},
\end{align}
where in (a) we used the fact that for every $j\in[d]$, we have $|x_j|\leq a$ and $|x'_j|\leq a$.
\item Fix an arbitrary $\bx\in\calB_{\infty}^{d}$.
\begin{align*}
\text{Unbiasedness:} \quad \bbE\left[\calR_\infty\left(\bx\right)\right]&=\frac{1}{d}\sum_{j=1}^{d} \be_j ad\left(\frac{e^{\eps_0}+1}{e^{\eps_0}-1}\right)\left(\frac{x_j}{a}\frac{e^{\eps_0}-1}{e^{\eps_0}+1}\right)\\
&=\sum_{j=1}^{d}\be_j x_j \\
&=\bx
\end{align*}
\begin{align*}
\text{Bounded variance:}\quad \bbE\|\calR_\infty(\bx)-\bx\|_{2}^{2} &\leq \bbE\|\calR_{\infty}(\bx)\|^2 = \bbE[\calR_{\infty}(\bx)^T\calR_{\infty}(\bx)] \notag \\
&=\frac{1}{d}\sum_{j=1}^d a^2 d^{2}\left(\frac{e^{\eps_0}+1}{e^{\eps_0}-1}\right)^{2} \notag \\
&= a^2 d^{2}\left(\frac{e^{\eps_0}+1}{e^{\eps_0}-1}\right)^{2}
\end{align*}
\end{enumerate}
This completes the proof of Lemma~\ref{R_infty-quantize-properties}.
\end{proof}

\end{document}